\documentclass[10pt,twocolumn,letterpaper]{article}

\usepackage[pagenumbers]{cvpr} 

\usepackage{graphicx}
\usepackage{amsmath}
\usepackage{amssymb}
\usepackage{booktabs}
\usepackage[toc]{appendix}
\usepackage{float}
\usepackage[accsupp]{axessibility}

\usepackage{comment}
\usepackage{amsmath,amssymb} 
\usepackage{color}

\usepackage{caption}
\usepackage{subcaption}
\usepackage{mathtools}

\usepackage{amsthm}
\newtheorem{theorem}{Theorem}
\newtheorem{corollary}{Corollary}[theorem]
\theoremstyle{definition}
\newtheorem{definition}{Definition}[section]

\newcommand{\Z}{\mathbb{Z}}

\newcommand{\N}{\mathbb{N}}

\newcommand{\R}{\mathbb{R}}

\newcommand{\Real}{\mathbb{R}}

\DeclareMathOperator\supp{supp}
\DeclareMathOperator\WF{WF}

\usepackage{xcolor}

\definecolor{tabred}{RGB}{214,39,40}
\definecolor{taborange}{RGB}{255, 127, 14}
\definecolor{tabgreen}{RGB}{44, 160, 44}
\definecolor{tabblue}{RGB}{31,119,180}

\definecolor{wrongultramarine}{rgb}{0.07, 0.04, 0.56}

\newcounter{RonCounter}

\newcounter{StefanCounter}

\newcounter{RWCounter}

\newcounter{HectorCounter}

\usepackage{hyperref}
\hypersetup{colorlinks,breaklinks}

\usepackage[capitalize]{cleveref}
\crefname{section}{Sec.}{Secs.}
\Crefname{section}{Section}{Sections}
\Crefname{table}{Table}{Tables}
\crefname{table}{Tab.}{Tabs.}

\def\cvprPaperID{6899} 
\def\confName{CVPR}
\def\confYear{2023}

\usepackage[font=small,labelfont=bf,tableposition=top]{caption}

\begin{document}

\title{Explaining Image Classifiers with Multiscale Directional Image Representation}

\author{Stefan Kolek$^{1}$, Robert Windesheim$^{1}$, Hector Andrade-Loarca$^{1}$, Gitta Kutyniok$^{1,2}$, Ron Levie$^{3}$\\
$^{1}$Ludwig-Maximilians-Universität München, Department of Mathematics\\
$^{2}$University of Tromsø, Department of Physics and Technology\\
$^{3}$Technion-Israel Institute of Technology, Department of Mathematics\\
{\tt\small \{kolek$,$windesheim$,$andrade$,$kutyniok\}@math.lmu.de}, {\tt\small levieron@technion.ac.il}
}


\maketitle

\begin{abstract}

Image classifiers are known to be difficult to interpret and therefore require explanation methods to understand their decisions. We present ShearletX, a novel mask explanation method for image classifiers based on the shearlet transform -- a multiscale directional image representation. Current mask explanation methods are regularized by smoothness constraints that protect against undesirable fine-grained explanation artifacts. However, the smoothness of a mask limits its ability to separate fine-detail patterns, that are relevant for the classifier, from nearby nuisance patterns, that do not affect the classifier. ShearletX solves this problem by avoiding smoothness regularization all together, replacing it by shearlet sparsity constraints. The resulting explanations consist of a few edges, textures, and smooth parts of the original image, that are the most relevant for the decision of the classifier. To support our method, we propose a mathematical definition for explanation artifacts and an information theoretic score to evaluate the quality of mask explanations.  We demonstrate the superiority of ShearletX over previous mask based explanation methods using these new metrics, and present exemplary  situations where separating fine-detail patterns allows explaining phenomena that were not explainable before.
\end{abstract}

\section{Introduction}\label{sec:intro}
\begin{figure}[t]
      \centering
  \begin{minipage}{1.0\linewidth}
  \centering
    \includegraphics[width=1.0\linewidth]{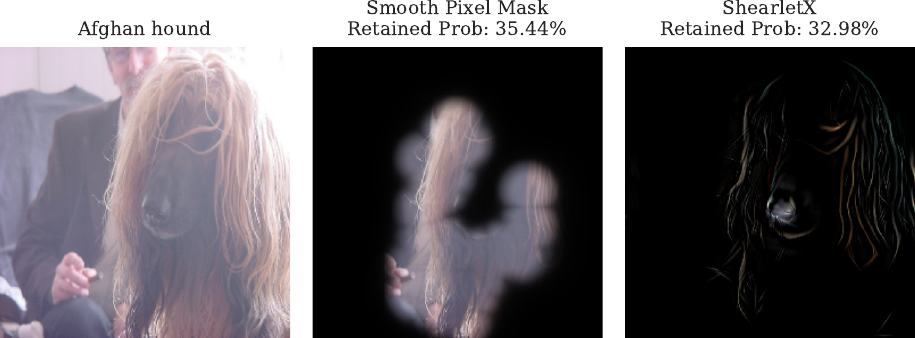}
    \vspace{-.3cm}
    \end{minipage}
  \begin{minipage}{1.0\linewidth}
  \centering
    \includegraphics[width=1.0\linewidth]{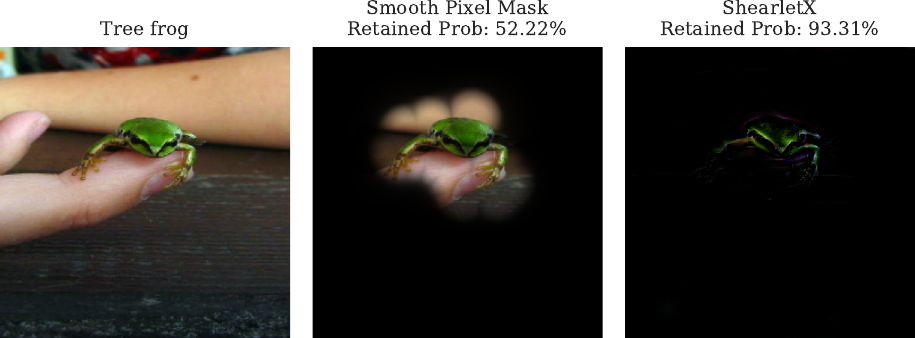}
    \vspace{-.3cm}
    \end{minipage}
  \begin{minipage}{1.0\linewidth}
  \centering
    \includegraphics[width=1.0\linewidth]{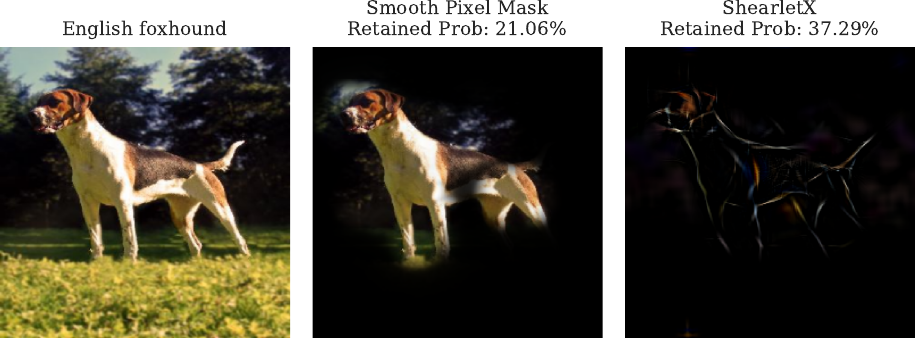}
    \end{minipage}

    \caption{Left column: ImageNet samples with prediction. Middle column: Smooth pixel mask explanation from Fong et al. \cite{Fong_2019_ICCV}. Right column: ShearletX (ours). Retained probability is computed as class probability after masking divided by class probability before masking. 
    ShearletX is the first mask explanation method that can separate fine-detail patterns, that are relevant for the classifier, from nearby patterns that are irrelevant, without producing artifacts.
    }\label{fig:teaser}
    \vspace{-.5cm}
\end{figure}

Modern image classifiers are known to be difficult to explain. Saliency maps comprise a well-established explainability tool that highlights important image regions for the classifier and helps interpret classification decisions. An important saliency approach  frames saliency map computation as an optimization problem over masks \cite{RDE_original_2019, in_distribution_cosmas_2020, kolek2022cartoon,  chang_GAN_explanation_2018, Fong_2019_ICCV, Fong2017InterpretableEO,Dabkowski2017}. The explanation mask is optimized to keep only parts of the image that suffice to retain the classification decision. However, Fong and Vedaldi \cite{Fong2017InterpretableEO} showed that an unregularized explanation mask is very susceptible to explanation artifacts and  is hence unreliable. Therefore, current practice \cite{Fong2017InterpretableEO,Fong_2019_ICCV,chang_GAN_explanation_2018} heavily regularizes the explanation masks to be smooth. The smooth explanation masks can communicate  useful explanatory information by roughly localizing the relevant  image region. However, the pattern that is relevant for the classifier is often overlaid on  patterns that do not affect the classifier. In such a situation the mask cannot effectively separate the relevant pattern from the nuisance pattern, due to the smoothness constraints. As a result, many details that are irrelevant to the classifier, such as background elements, textures, and other spatially localized patterns, appear in the explanation.

An ideal mask explanation method should be resistant to explanation artifacts and capable of highlighting only relevant patterns.
We present such a method, called \emph{ShearletX}, that is able to separate different patterns that occupy nearby spatial locations by optimizing a mask in the shearlet representation  of an image \cite{kutyniok2012shearlets}. Due to the ability of shearlets to efficiently encode directional features in images, we can separate relevant fine-grained image parts, like edges, smooth areas, and textures, extremely well. We show both theoretically and experimentally that defining the mask in the shearlet domain circumvents explanation artifacts.  The masked image is optimized so that the classifier retains its prediction as much as possible and to have small spatial support (but not high spatial smoothness), while regularizing the mask to be sparse in the shearlet domain. This regularization assures that ShearletX retains only relevant parts, a fact that we support by a new information theoretic score for the quality of mask explanations. 
Figure \ref{fig:teaser} gives examples demonstrating that ShearletX can separate relevant details from nuisance patterns, which smooth pixel masks cannot. Our contributions are summarized as follows:
\begin{enumerate}
    \item ShearletX: The first mask explanation method that can effectively
     separate fine-detail patterns, that are relevant for the classifier, from nearby nuisance patterns, that do not affect the classifier. 
     \vspace{-.2cm}
     \item Artifact Analysis: Our explanation method is based on low-level vision for maximal  interpretability and  belongs to the family of methods that produce out-of-distribution explanations. To validate that the resulting out-of-distribution explanations are meaningful, we develop a theory to analyze and quantify explanation artifacts, and prove that ShearletX is resilient to such artifacts.  
    \vspace{-.2cm}
    \item Hallucination Score:  a new metric for mask explanations that quantifies explanation artifacts by measuring the amount of  edges in the explanation that do not appear in the original image.
    \vspace{-.2cm}
    \item Concisesness-Preciseness Score: A new information theoretic metric for mask explanations that  gives a high score for explanations that extract the least amount of information from the image to retain the classification decision as accurately as possible.
   \vspace{-.2cm}
    \item Experimental Results: We demonstrate that ShearletX performs better than previous mask explanations using our new metrics and give examples where ShearletX allows to explain phenomena that were not explainable with previous saliency methods.
\end{enumerate}
The source code for the experiments is publicly available \footnote{\url{https://github.com/skmda37/ShearletX}}.

\section{Related Work}\label{sec:related work}
The explainability field has experienced a surge in research activity over the last decade, due to the societal need to explain machine learning models. We focus on explainability aspects of image classifiers, where saliency maps provide an important and useful way of understanding a classifier's prediction. The community has also introduced  other tools, such as concept-based methods \cite{Kim2018InterpretabilityBF} and inherently interpretable architectures \cite{rudin2019looks}, but we will not focus on these in our work. In the following, we review previously introduced saliency map methods. 

\subsubsection*{Pixel Attribution Methods}
Many saliency map methods assign a relevance score to each pixel indicating its relevance for the prediction. Such methods include Gradient Saliency \cite{Simonyan2014DeepIC}, Grad-CAM\cite{Selvaraju2019GradCAMVE}, LRP \cite{Layerwise_relevance_prop2015}, Guided Backprop \cite{guided_backprop_2015}, and Integrated Gradients \cite{Integrated_gradient_2017_sundararajan}. Although these methods can help explain classifiers, they are heuristic in their approach and not  optimized for a well-defined notion of relevance. Therefore, the fidelity of pixel attribution methods needs to be checked post-hoc with metrics such as the area over the perturbation curve \cite{Arras2017ExplainingRN} and can be low. Moreover, Kindermans et al. \cite{alber2019reliability} showed that pixel attribution methods can be highly unreliable. Other well-known explanation methods, such as LIME \cite{Ribeiro2016WhySI} and SHAP \cite{NIPS2017_7062} can be applied to images, by first segmenting the image into superpixels and assigning a relevance score to each superpixel. However, research recently revealed various vulnerabilities of LIME and SHAP \cite{Slack2020}. 

\subsubsection*{Pixel Mask Explanations}
Mask explanations do not attribute individual relevance scores to (super)pixels but rather optimize a mask to delete as much information of the image as possible while retaining the classifier's prediction. The advantage of this approach is that one optimizes for a natural interpretability objective that can be quickly validated in two steps: (1) Determining which and how much information was deleted by the mask (2) Computing the class probability score after masking the image. Fong and Vedaldi \cite{Fong2017InterpretableEO} were the first to find an explanation mask as a solution to an optimization problem that can be summarized as 
\begin{equation}
    \max_{m\in\mathcal{M}}\; \mathop{\mathbb{E}}_{u\sim\nu} \Big[\Phi_c(x\odot m + (1-m)\odot u)\Big] - \lambda\cdot\|m\|_1,
\end{equation}
where $x\in\R^d$ is the input image, $\Phi_c$ returns the classifier's class probability, $u\in\R^d$ is a random perturbation from a predefined probability distribution $\nu$ (\eg, constant, blur, or noise), $m\in\mathbb{R}^d$ is a mask on $x$, $\lambda\in\R_+$ is the Lagrange multiplier encouraging sparsity in $m$, and $\mathcal{M}$ is a prior over the explanation masks. Fong and Vedaldi \cite{Fong2017InterpretableEO} found that not choosing a prior, \ie $\mathcal{M} = [0,1]^d$, produces explanation artifacts. To mitigate artifacts, they enforce a more regular structure on the mask by using an  upsampled lower resolution mask and regularizing the mask's total variation (TV). Fong et al. \cite{Fong_2019_ICCV} improved this method by reformulating the area constraint  and adding a new parametric family of smooth masks, which allowed to remove all hyperparameters from the optimization problem. The masks remain extremely smooth but the main advantage is that the size of the mask can be controlled by an area constraint that directly controls the size of the mask as a percentage of the total image area. We will refer to this method  as \emph{smooth pixel mask} to highlight the fact that this method produces extremely smooth explanations due to strong smoothness constraints on the mask.

\subsubsection*{Wavelet Mask Explanations}
Kolek et al. \cite{kolek2022cartoon} 
 proposed the \emph{CartoonX} method, which  masks in the wavelet representation of images to extract the relevant piece-wise smooth part of an image. Wavelets sparsely represent piece-wise smooth images and therefore the wavelet sparsity constraint in CartoonX typically leads to piece-wise smooth explanations. However, Kolek et al. \cite{kolek2022cartoon} do not compare CartoonX to smooth pixel masks \cite{Fong_2019_ICCV}, which also enforce piece-wise smoothness by regularizing and parameterizing the pixel mask. Besides lacking a clear advantage over smooth pixel masks, we find that CartoonX produces blurry spatial areas that can be quite difficult to interpret (see Figure \ref{fig:cartoonx vs waveletx vs shearletx}). ShearletX improves upon CartoonX by (a) leveraging the advantages of shearlets over wavelets for representing edges in images, (b) eliminating an ambiguous spatial blur in CartoonX, and (c) having a clear advantage over smooth pixel masks.

\section{Background}\label{sec:background}
To develop and analyze ShearletX we need to first give the necessary technical background for wavelets \cite{mallat1999wavelet} and shearlets \cite{kutyniok2012shearlets} in the context of images. 

\subsubsection*{Wavelets for Images}
A gray-level image can be mathematically modeled as a square integrable function $f:\R^2\to\R$. 
A wavelet $\psi: \R^2 \to \R$ is a spatially localized bump with oscillations, that is used to probe the local frequency, or scale, of an image.  Three suitably chosen mother wavelets $\psi^1,\psi^2,\psi^3\in L^2(\R^2)$ with dyadic dilations and translations yield an orthonormal basis 
\begin{equation}
    \Big\{\psi^k_{j,n}\coloneqq\frac{1}{2^j}\psi^k\Big(\frac{\cdot - 2^jn}{2^j}\Big)\Big\}_{j\in\Z, n\in\Z^2,1\leq k\leq 3}
\end{equation}
of the square integrable function space $L^2(\R^2)$. The three indices $k\in\{1,2,3\}$ correspond to vertical, horizontal, and diagonal directions. The image $f$ can be probed in direction $k\in\{1,2,3\}$, at location $n\in\Z^n$, and at scale $2^j\in\Z$ by taking the inner product $\langle f,\psi^k_{j,n}\rangle$, which is called a \emph{wavelet (detail) coefficient}. The wavelet coefficient $\langle f,\psi^k_{j,n}\rangle$ has high amplitude if the image $f$ has sharp transitions over the support of $\psi^k_{j,n}$. Pairing $\psi^1,\psi^2,\psi^3\in L^2(\R^2)$, with an appropriate scaling function $\phi\in L^2(\R^2)$, defines a multiresolution approximation. More precisely, for all $J\in\Z$, any finite energy image $f$ decomposes into 
\begin{equation}
    f = \sum_{n\in\Z^2}a_n \phi_{J,n} + \sum_{1\leq k\leq 3}\sum_{j\leq J, } d^k_{j,n}\psi^k_{j,n},\label{eq: basis equation}
\end{equation}
where $a_n=\langle f, \phi_{J,n}\rangle$ and $d^k_{j,n}=\langle f, \psi^k_{j,n}\rangle$ are the approximation coefficients at scale $J$ and wavelet coefficients at scale $j-1$, respectively. In practice, images are discrete signals $x[n_1,n_2]$ with pixel values at discrete positions $n=(n_1,n_2)\in \Z^2$ but they can be associated with a function $f\in L^2(\R^2)$ that is approximated at some scale $2^L$ by $x$. The discrete wavelet transform (DWT) of an image $x$ then computes an invertible wavelet image representation
\begin{equation}
  \mathcal{DWT}(x) = \Big\{a_{J,n}\Big\}_n \cup \Big\{ d^1_{j,n}, d^2_{j,n} d^3_{j,n}\Big\}_{L<j\leq J,n}
\end{equation}
corresponding to discretely sampled approximation and wavelet coefficients of $f$.

\subsubsection*{Shearlets for Images}\label{subsec:shearlets for images}
Wavelets are optimal sparse representations for signals with point singularities \cite{devore_1998}, in particular, piece-wise smooth 1d signals. However, images are 2d signals where many singularities are  edges, which are anisotropic (directional), and are not optimally represented by wavelets. Shearlets \cite{kutyniok2012shearlets} extend wavelets and form a multiscale directional representation of images, which allows efficient encoding of anisotropic features. Next, we describe the continuous shearlet system, and note that the discrete shearlet system is just a discrete sampling of the continuous system. The shearlet transform was introduced in \cite{gitta2005shearlets}. Similarly to  the wavelet transform, the shearlet transform applies transformations to a function, called the mother shearlet, to generate a filter bank. The transformations are (a) translation, to change the location of the shearlet probe, (b) anisotropic dilation, to change the scale and shape, creating elongated probes of different scales, and  (c) shearing, to probe at different orientations. To dilate and shearing a function, we define the following three matrices:
\begin{equation*}
    A_a\coloneqq \begin{pmatrix}
        a & 0 \\
        0 & \sqrt{a}
    \end{pmatrix},
    \quad
    \widetilde{A}_a\coloneqq \begin{pmatrix}
        \sqrt{a} & 0 \\
        0 & a
    \end{pmatrix},
    \quad
    S_s \coloneqq \begin{pmatrix}
        1 & s \\
        0 & 1
    \end{pmatrix},
\end{equation*}
where $s,a\in \Real$.
Given $(a,s,t) \in \Real_+ \times \Real \times \Real^2$, $\psi \in L^2(\Real^2)$, and $x\in \Real^2$, we define
\begin{equation}\label{eq:shearlets}
\begin{aligned}
\psi_{a,s,t, 1}(x) &\coloneqq a^{-\frac{3}{4}} \psi\left(A_a^{-1}S_s^{-1} (x-t)\right),\\
\psi_{a,s,t, -1}(x) &\coloneqq a^{-\frac{3}{4}} \widetilde{\psi}\left(\widetilde{A}_a^{-1}{(S_s^{T})}^{-1} (x-t)\right),
\end{aligned}
\end{equation}
where $\widetilde{\psi}(x_1,x_2) \coloneqq \psi(x_2,x_1)$, for all $x = (x_1,x_2) \in \Real^2$, and $\psi$ is the mother shearlet. The continuous shearlet transform is then defined as follows.
\begin{definition}[Continuous Shearlet Transform]
Let $\psi \in L^2(\Real^2)$. Then the family of functions $\psi_{a,s,t, \iota} \colon \Real^2 \to \Real$ parametrized 
by $(a,s,t, \iota) \in \Real^+ \times \Real \times \Real^2 \times \{-1,1\}$ that are defined in \eqref{eq:shearlets} is called a \emph{shearlet system}.
The corresponding  shearlet transform is defined by
\begin{equation}
        \mathcal{SH}_\psi: L^2(\Real^2) \to L^\infty\bigl(\Real^+ \times \Real \times \Real^2 \times \{-1,1\}\bigr), \label{shearlet transform}
\end{equation}
where $\mathcal{SH}_\psi(f)(a,s,t,\iota) \coloneqq \langle f, \psi_{a,s,t,\iota}\rangle$.

\end{definition}
 The continuous shearlet transform can be digitized to the \emph{digital shearlet transform}\footnote{For the digital shearlet transform, we used pyshearlab from  \url{http://shearlab.math.lmu.de/software\#pyshearlab}.}  \cite{digitalshearlet2016}, denoted as $\mathcal{DSH}$, by sampling a discrete system from the function system (\ref{shearlet transform}). Note that the digital shearlet transform, like the discrete wavelet transform, is an invertible transformation.

\section{Method}\label{sec:method}

In this section, we develop our novel mask explanation method \emph{ShearletX} (Shearlet Explainer).

\subsubsection*{ShearletX}\label{subsec:ShearletX Method}
\begin{figure}[t]

  \centering
  \begin{minipage}{1.0\linewidth}
  \centering
    \includegraphics[width=1.\linewidth]{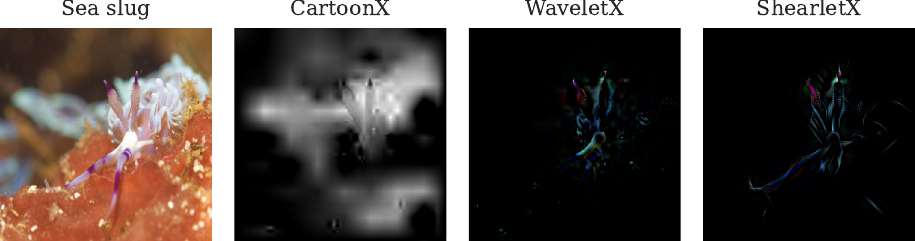}
    \vspace{-.2cm}
    \end{minipage}
  \begin{minipage}{1.0\linewidth}
  \centering
    \includegraphics[width=1.\linewidth]{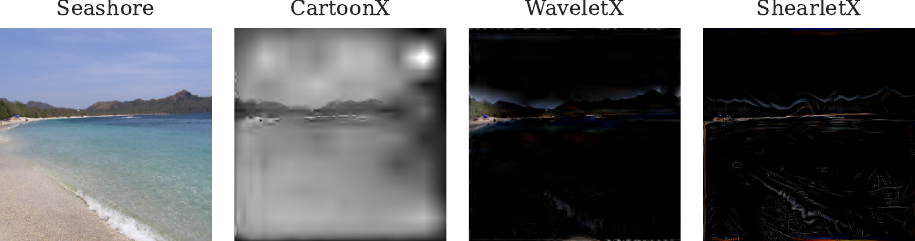}
    \vspace{-.2cm}
    \end{minipage}
  \begin{minipage}{1.0\linewidth}
  \centering
    \includegraphics[width=1.\linewidth]{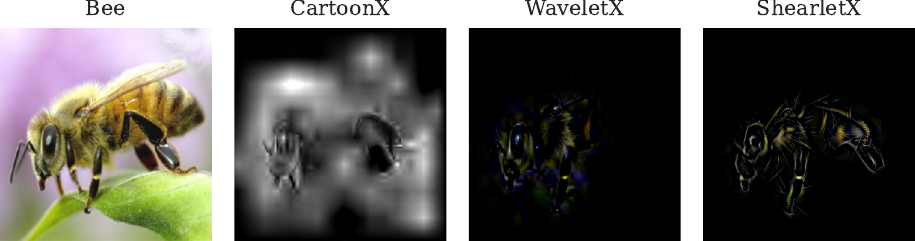}
    \end{minipage}

   \caption{Left Column: Input images classified by VGG-19 \cite{vgg19}. Comparing CartoonX by Kolek et al. \cite{kolek2022cartoon}, WaveletX (ours), and ShearletX (ours). WaveletX improves CartoonX significantly due to the spatial penalty that eliminates undesirable blurry spatial areas that are difficult to interpret.  Note that ShearletX represents relevant edges better and produces much crisper explanations than WaveletX. This is because anisotropic features, such as edges, can be encoded more efficiently with shearlets than with wavelets.
   }
   \label{fig:cartoonx vs waveletx vs shearletx}
\end{figure}

The optimization objective for ShearletX is 
\begin{align}
    \max_{m} \;&\mathop{\mathbb{E}}_{u\sim \nu}\Big[\Phi_c(\mathcal{DSH}^{-1}(m\odot \mathcal{DSH}(x) + (1-m)\odot u))\Big] \nonumber\\
    \;&- \lambda_1 \|m\|_1 - \lambda_2  \|\mathcal{DSH}^{-1}(m\odot \mathcal{DSH}(x))\|_1, \label{eq: shearletx objective}
\end{align}
where $m\in[0,1]^n$ denotes a mask on the digital shearlet coefficients, $\Phi_c$ returns the class probability of the classifier, $\nu$ is the perturbation distribution, $\lambda_1\in\R_+$ controls the sparseness of the shearlet mask, and $\lambda_2\in\R_+$ controls the penalty for spatial energy.  The final ShearletX explanation is given by masking the shearlet coefficients and inverting the masked shearlet coefficients back to pixel space, \ie,
\begin{equation}
\text{ShearletX}(x)\coloneqq\mathcal{DSH}^{-1}(m\odot \mathcal{DSH}(x)).    
\end{equation}

The expectation term in the ShearletX objective (\ref{eq: shearletx objective}) ensures that the image after masking and perturbing retains the classification decision. We find that the spatial penalty is a crucial technical addition, that deletes classifier irrelevant spatial energy in the explanation and ensures that no irrelevant blurry areas remain in the explanation, as opposed to CartoonX \cite{kolek2022cartoon} (see Figure \ref{fig:cartoonx vs waveletx vs shearletx}). A smooth area is retained by ShearletX only if it is important for the classifier (see English Foxhound and Frog in Figure \ref{fig:teaser}). Moreover, the color can be distorted if the original color is not important. When the color is important for the classifier, ShearletX will keep the color (see, for example, Figure \ref{fig:teaser}, where ShearletX keeps the brown color of the English Foxhound's head and the green color of the Frog).

For the perturbation distribution $\nu$, we deliberately avoid in-distribution perturbations from an in-painting network, as opposed to \cite{chang_GAN_explanation_2018}. The reason is that in-distribution masks may delete parts of the image that are relevant to the classifier if the in-painter in-fills such parts effectively, making the explanation hard to interpret. 
Therefore, we follow the out-of-distribution approach of CartoonX \cite{kolek2022cartoon}, and use white noise in the representation system  that is adapted to the mean and variance of the shearlet coefficients (see Supplementary Material \ref{Appendix: implementation details} for details).

\subsubsection*{WaveletX}
Solely adding the spatial penalty to the CartoonX objective yields significantly better explanations than the original CartoonX method and eliminates the undesirable blurry areas, that are difficult to interpret (see Figure \ref{fig:cartoonx vs waveletx vs shearletx}). We will refer to this new method as WaveletX to highlight the fact that WaveletX and ShearletX only differ in the choice of the representation system. The WaveletX optimization objective is
\begin{align}
    \max_{m} \;&\mathop{\mathbb{E}}_{u\sim \nu}\Big[\Phi_c(\mathcal{DWT}^{-1}(m\odot \mathcal{DWT}(x) + (1-m)\odot u))\Big] \nonumber\\
    \;&- \lambda_1 \|m\|_1 - \lambda_2  \|\mathcal{DWT}^{-1}(m\odot \mathcal{DWT}(x))\|_1,
\end{align}
where $\mathcal{DWT}$ denotes the discrete wavelet transform of images. Note that we recover the CartoonX \cite{kolek2022cartoon} objective if we set $\lambda_2=0$. In Figure \ref{fig:cartoonx vs waveletx vs shearletx}, we compare CartoonX \cite{kolek2022cartoon}, WaveletX, and ShearletX on examples classified by a VGG-19\cite{vgg19} network trained on ImageNet \cite{deng2009imagenet}.

\section{Theory}\label{sec:theory}
Fong and Vedaldi \cite{Fong2017InterpretableEO} first observed the problem of explanation artifacts for mask explanations. We identify explanation artifacts as artificial edges in the explanation (see Figure \ref{fig:vis artificial edges}).   Artificial edges can form patterns that activate the class label but are not present in the original image. Therefore, a good mask explanation method should not be able to form many artificial edges.  In this section, we show theoretically that ShearletX and WaveletX are not prone to artificial edges, by proving that the continuous counterparts of WaveletX and ShearletX cannot create edges that are not present in the original image.

\subsubsection*{ShearletX is Resistant to Edge Artifacts}

In this section, we prove that ShearletX applied to continuous images cannot create artificial edges. 
 When working with shearlets, it is common to model edges as the \emph{wavefront} set of a continuous image \cite{AndradeLoarca2019ExtractionOD,andrade2020shearlets}. The wavefront set is a concept that characterizes the oriented singularities of distributions, in particular, of $L^2(\R^2)$ functions. We state the mathematical definition of the wavefront set below and provide an intuitive explanation afterwards since the definition is somewhat technical.
\begin{definition}{\cite[Section 8.1]{AnLinPDOHoermander}}\label{def:WaveFrontSet}
Let $f \in L^2(\R^2)$ and $k \in \N$. A point $(x, \lambda)\in \R^2 \times \mathbb{S}^1$ is a \emph{$k$-regular directed point of $f$} if there exist open neighbourhoods $U_x$ and $V_\lambda$ of $x$ and $\lambda$, respectively, and a smooth function $\phi\in C^\infty(\R^2)$ with $\supp \phi \subset U_x$ and $\phi(x) = 1$ such that
\begin{equation*}
    \bigl| \widehat{\phi f} (\xi) \bigr|
    \leq C_k \bigl( 1+|\xi| \bigr)^{-k}
    \quad \text{$\forall$ $\xi \in \R^2 \setminus \{0\}$ s.t. $\xi /|\xi| \in V_\lambda$}
\end{equation*}
holds for some $C_k >0$, where $\widehat{f}$ denotes the Fourier transform of $f$. The \emph{$k$-wavefront set} $\WF_k(f)$ is the complement of the set of all $k$-regular directed points and the \emph{wavefront set} is defined as $\WF(f) \coloneqq \bigcup_{k \in \N} \WF_k(f)$.
\end{definition}
The wavefront set defines the directional singularities of a function $f$ via the Fourier decay of local patches of the function.
For piece-wise smooth images with discontinuities along smooth curves, the wavefront set is exactly the set of edges with the orientation of the edge. This explains why the wavefront set is a good model for edges. The wavefront set of an image can be completely determined by the decay properties of its shearlet coefficients \cite{grohs2011continuous}. 
More precisely, the regular point-direction pairs of an image (the complement of the wavefront set) are exactly the pairs of locations and directions where the shearlet coefficients exhibit rapid decay as $a \to 0$  (the precise statement can be found in Supplementary Material \ref{Appendix Theory}). 
We use this property of shearlets to prove that ShearletX cannot produce artificial edges for continuous images.
\begin{theorem}\label{thm:shearletx}
Let $x\in L^2[0,1]^2$ be an image modeled as a $L^2$-function. Let $m$ be a bounded mask on the shearlet coefficients of $x$ and let $\hat x$ be the image $x$ masked in shearlet space with mask $m$. Then, we have  $\textrm{WF}(\hat x)\subset \textrm{WF}(x)$ and thus masking in shearlet space does not create new edges.

\end{theorem}
 
The idea behind the proof is that creating artificial singularities in regular point-directions of the image would require creating asymptotically slower shearlet decay by masking the coefficients. This is impossible, as masking can only increase the decay rate. See Supplementary Material \ref{Appendix Theory} for a full proof of Theorem \ref{thm:shearletx}.  

While in the real world images are digital,  they are still an approximation of continuous images that becomes better with increasing resolution. In Section \ref{sec:experiments}, we show experimentally that Theorem \ref{thm:shearletx} indeed predicts the behavior of masked digital images, and ShearletX is not susceptible to explanation artifacts.

\subsubsection*{WaveletX is Resistant to Edge Artifacts}
When analyzing WaveletX, we opt to model singularities via local Lipschitz regularity instead of using the wavefront set approach. This approach is preferable since the Lipschitz regularity of a function is completely characterized by the rate of decay of its wavelet coefficients, as the scale goes to zero \cite[Theorem 9.15]{mallat_sparse_edition}. We hence define a regular point as a point for which the image is $\alpha$-Lipschitz regular in a neighborhood of the point, with $\alpha\geq 1$ (see Definition \ref{def: lipschitz regularity} in the supplementary material for Lipschitz regularity, and, in particular, the definition for $\alpha> 1$). A singular point is a point which is not regular. Singular points describe image elements such as edges and point singularities.

 \begin{theorem}[Informal version of Theorem \ref{theorem: waveletx formal version}]\label{thm:waveletx}
Let $x\in L^2[0,1]^2$ be an image modeled as a $L^2$-function. Masking the wavelet coefficients of $x$ with a bounded mask cannot create new singularities.
\end{theorem}
The above theorem is an informal version of our formal Theorem \ref{theorem: waveletx formal version} in Supplementary Material \ref{Appendix Theory}. Similarly to ShearletX, Theorem \ref{thm:waveletx} predicts the behavior for digital images well, and WaveletX is not prone to produce explanation artifacts.

\section{Explanation Metrics for Mask Explanations}\label{sec:explanation metrics for mask explanations}
We now propose two new explanation metrics for mask explanations: (1) The \emph{conciseness-preciseness} (CP) score to evaluate the preciseness of a mask explanation adjusted for its conciseness (2) The \emph{hallucination score} to quantify explanation artifacts.

\begin{figure}[t]

  \centering
  \begin{minipage}{1.0\linewidth}
  \centering
    \includegraphics[width=1.\linewidth]{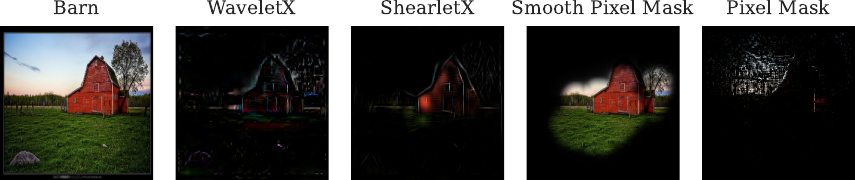}
    \vspace{-.25cm}
    \end{minipage}
  \begin{minipage}{1.0\linewidth}
  \centering
    \includegraphics[width=1.\linewidth]{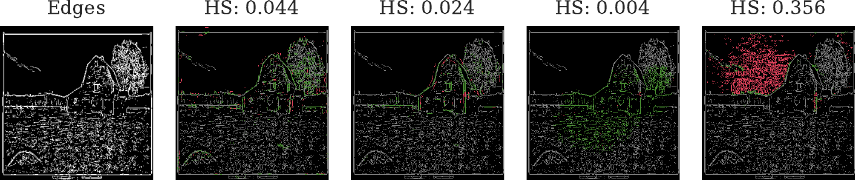}
    \end{minipage}

   \caption{First row: image correctly classified as barn, WaveletX, ShearletX,  smooth pixel mask by Fong et al. \cite{Fong2017InterpretableEO}, and pixel mask without smoothness constraints. Second row visualizes the edges in the image and all explanations. Edges marked in red are artificial and quantified by the hallucination score (HS). Green edges are present in the original image. The pixel mask without smoothness constraints hallucinates an artificial barn, which is an example of an explanation artifact and results in a very high HS.}
   \label{fig:vis artificial edges}
\end{figure}

\subsubsection*{Conciseness and Preciseness}
Several metrics, such as remove and retrain (ROAR) \cite{roar2019} and area over the perturbation curve (AOPC) \cite{Arras2017ExplainingRN}, have been proposed to quantify the fidelity of saliency maps. However, these metrics are designed for pixel attribution methods that provide an ordering of feature importance. Mask explanations can be 
immediately evaluated by simply plugging in the masked image into the classifier and checking the class probability. A good mask explanation retains the class probability of the prediction. We refer to this property as \emph{preciseness} of the explanation. However, a good explanation mask should not only be precise but also \emph{concise}, \ie, the mask should extract the least amount of information from the available pool of data. We introduce a class of new explanation metrics that combine both aspects into one metric, which we call  \emph{conciseness-preciseness} (CP) scores. The definition is
\begin{equation}
    \text{CP} = \frac{\text{Retained Class Probability}}{\text{Retained Image Information}}.
\end{equation}
The retained class probability (preciseness) is computed as the class probability after masking divided by the class probability before masking.  We compute the retained information of the image (conciseness) as the information of the masked image divided by the information of the original image.
We experiment with three different ways of measuring the information of the image:
(1) CP-Entropy: The entropy in the respective image representation system (wavelet, shearlet, or pixel), (2) CP-$\ell_1$: The $\ell_1$-norm in the respective representation system (wavelet, shearlet, or pixel), (3) CP-$\ell_1$ Pixel: The $\ell_1$-norm in pixel space irrespective of representation system. For the CP-Entropy, we compute the retained image information of an image with representation coefficients $\{c_i\}_i$ and mask $\{m_i\}_i$ as 
\begin{equation}
 \exp\big(H\{|m_ic_i|^2\}_i\big)/\exp\big( H\{|c_i|^2\}_i\big)\label{exponential entropy of image}
\end{equation}
where $H$ denotes the entropy of the induced probability distributions. We use here the exponential of the  entropy, also called the \emph{extent}  \cite{exponential-entropy}, to avoid dividing by zero in the zero-entropy case. For CP-$\ell_1$, we compute the retained information as the \emph{relative sparsity} $ \|\{m_ic_i\}_i\|_1 / \|\{c_i\}_i\|_1$. Note that by measuring information through entropy or $\ell_1$-norm in the respective representation system we normalize for the fact that shearlets and wavelets already represent images much more sparsely than pixel representations. The CP score can be interpreted as a measure of preciseness adjusted for by the conciseness of the explanation. Explanations with higher CP scores are superior, assuming no explanation artifacts, which we measure with another metric that we define next.

\begin{figure*}[th]
     \centering
     \begin{subfigure}[b]{0.33\textwidth}
          \centering
  \includegraphics[width=.9\linewidth]{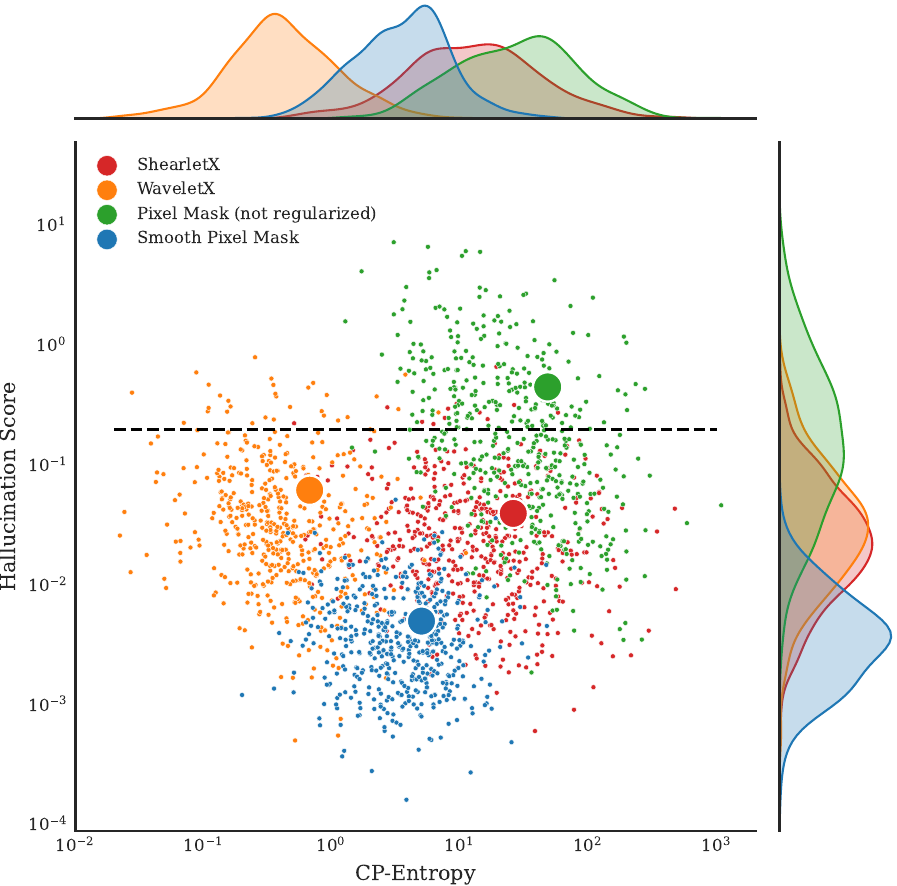}
         \caption{}
         \label{fig:y equals x}
     \end{subfigure}
     \hfill
     \begin{subfigure}[b]{0.33\textwidth}
         \centering
            \includegraphics[width=.9\linewidth]{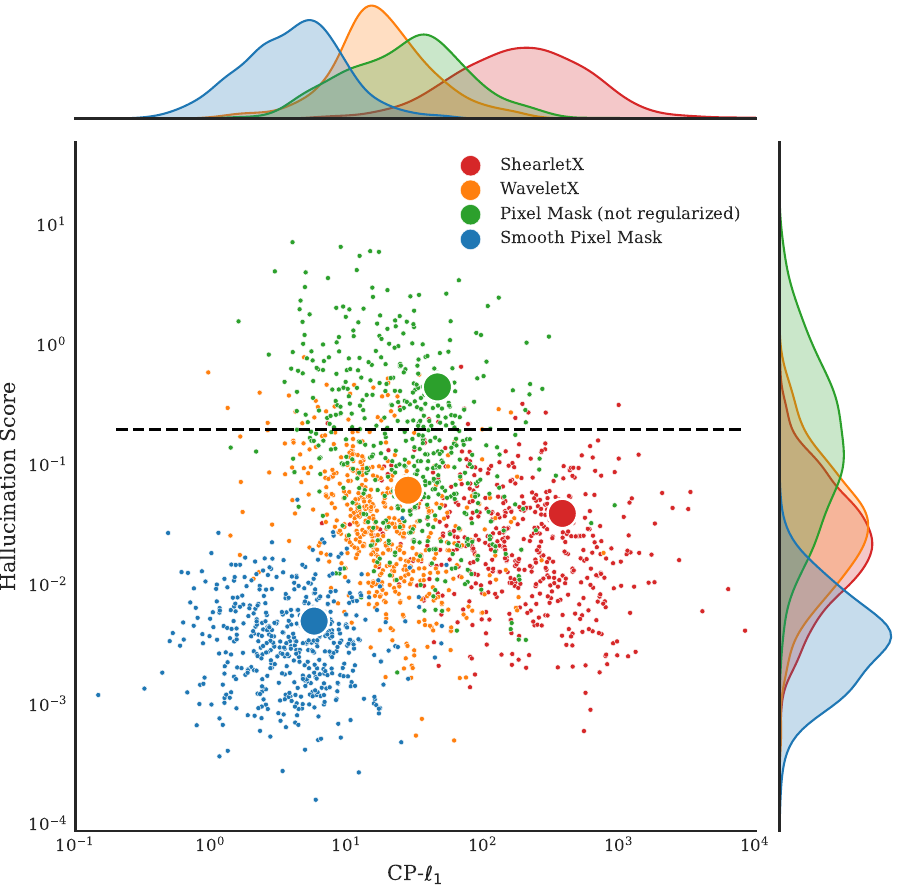}
         \caption{}
         \label{fig:three sin x}
     \end{subfigure}
     \begin{subfigure}[b]{0.33\textwidth}
          \centering
          \includegraphics[width=.9\linewidth]{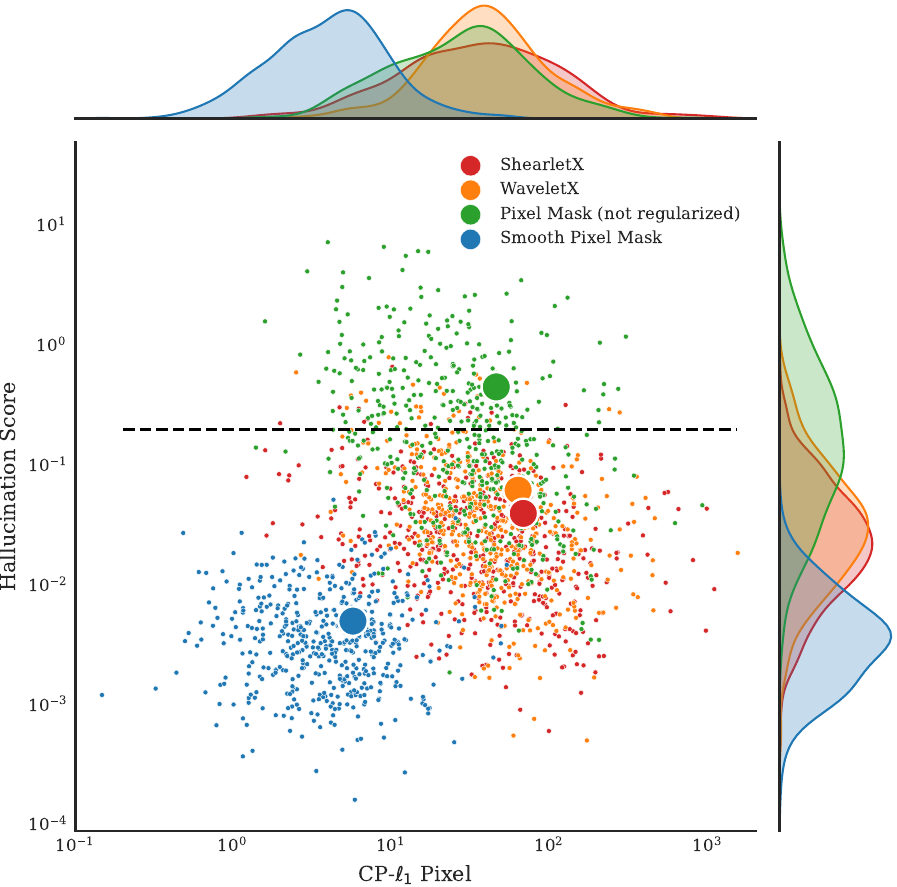}
         \caption{}
         \label{fig:y equals x}
     \end{subfigure}

        \caption{Scatter plot of hallucinaton score (lower is better) and conciseness-preciseness score (higher is better) for ShearletX, WaveletX, smooth pixel masks \cite{Fong_2019_ICCV}, and pixel mask without smoothness constraints. Retained information of an image for the CP score is measured (a) as entropy in respective representation, (b)  as $\ell_1$-norm in respective representation, (c) as $\ell_1$-norm in pixel space irrespective of representation. The black horizontal line marks  explanations where the artificial edges amount to 20\% of all edges in the original image. The mean hallucination and CP score are highlighted as big colored dots. ShearletX beats smooth pixel masks and WaveletX across all CP scores while having much better hallucination score than pixel masks without smoothness constraints.}
        \label{fig:scatter-plot}
\end{figure*}

\subsubsection*{Hallucination Score}
Artificial edges in a mask explanation are edges that are present after masking   that were not present in the original image. They can form artificial patterns, that appear as hallucinations (see the far right example in Figure \ref{fig:vis artificial edges}). The hallucinations can activate the class label, which is what the explanation optimized for, but do not actually explain the prediction. Therefore, artificial edges are undesirable and can lead to explanation artifacts. We propose to measure such explanation artifacts with a metric that we call \emph{hallucination score} (HS). We compute the hallucination score of an explanation as the number of edges that occur in the explanation but not in the original image, normalized by the number of edges that occur in the original image:
\begin{equation}
       \text{HS} = \frac{\#\Big(\text{Edges}(\text{Explanation})\setminus\text{Edges}(\text{Image})\Big)}{\#\text{Edges}(\text{Image})},
\end{equation}
where ``Explanation" refers to the image obtained after masking,  ``Image" refers to the original input image, and ``Edges" denotes an edge extractor that computes the set of pixels that belong to the edges of the input. Figure \ref{fig:vis artificial edges} provides an example for the hallucination score.

\section{Experiments}\label{sec:experiments}
In this section, we  experimentally show (a) that  ShearletX and WaveletX do not create a significant amount of artificial edges in practice and are thus resilient to explanation artifacts and (b) that ShearletX outperforms all other mask explanation methods in conciseness-preciseness scores.

\subsubsection*{Implementation}
We use the ImageNet \cite{deng2009imagenet} dataset and indicate in each experiment which classification model was used. For ShearletX, we optimize the shearlet mask with the Adam optimizer \cite{adam} for 300 steps on the ShearletX objective in (\ref{eq: shearletx objective}). We use a single shearlet mask for all three RGB channels as in \cite{kolek2022cartoon}.
The hyperparameter choice for ShearletX, WaveletX, smooth pixel masks, and pixel masks without smoothness constraints are discussed in detail in Supplementary Material \ref{Appendix: implementation details}.
We note as a limitation that, in practice, ShearletX is $5\times$ times slower than smooth pixel masks \cite{Fong_2019_ICCV} and WaveletX but not prohibitively slow for many applications (see Supplementary Material \ref{appendix: runtime} for runtime comparison). For details on the edge detector that we used for the hallucination score, see Supplementary Material \ref{Appendix: implementation details}.

\subsubsection*{Comparison of Mask Explanations}
We compute ShearletX, WaveletX, the smooth pixel mask by Fong et al. \cite{Fong2017InterpretableEO} (with area constraint 20\%), and the pixel mask without smoothness constraints for 500 random samples from the ImageNet validation dataset and compute the hallucination scores and conciseness-preciseness scores, which are plotted in Figure \ref{fig:scatter-plot}. We
use a ResNet18 \cite{He2016DeepRL} classifier but our results are consistent across different ImageNet classifiers and different area constraints (5\%, 10\%, and 20\%) for the smooth pixel mask (see Supplementary Material \ref{appendix: scatter plots}). 

The scatter plots in Figure \ref{fig:scatter-plot}  show that pixel masks without smoothness constraints
have extremely high hallucination scores, which confirms their proneness to explanation artifacts. The smooth pixel masks by Fong et al. \cite{Fong2017InterpretableEO} have almost no artificial edges (hallucination score very close to zero) because the masks are constrained to be extremely smooth. ShearletX and WaveletX have on average a moderately higher hallucination score than smooth pixel masks but their upper tail remains vastly lower than the tail for pixel masks without smoothness constraints (note the logarithmic scales in the scatter plots). In Figure \ref{fig:vis artificial edges}, one can see that a hallucination score in the order of $10^{-2}$ produces very few visible artificial edges. Therefore, we conclude from the scatter plot that ShearletX and WaveletX create very few artificial edges and are resilient to explanation artifacts. This also confirms that our Theorem \ref{thm:shearletx} and Theorem \ref{thm:waveletx} approximately hold for discrete images.  

Figure \ref{fig:scatter-plot} further shows that ShearletX has a significantly higher CP-Entropy, CP-$\ell_1$, and CP-$\ell_1$ Pixel score than the smooth pixel masks by Fong et al.
\cite{Fong2017InterpretableEO}.  This validates our claim that ShearletX can delete many irrelevant features that the smooth pixel masks cannot. 
ShearletX also outperforms WaveletX on all CP scores and even slightly on the hallucination score.  Pixel masks without smoothness constraints have the highest CP-Entropy score but are disqualified due to their unacceptable hallucination score. ShearletX, is the only method that has top performance on all CP scores.

\begin{figure*}[t]
  \centering
  \begin{minipage}{1.0\linewidth}
  \centering
    \includegraphics[width=1.\linewidth]{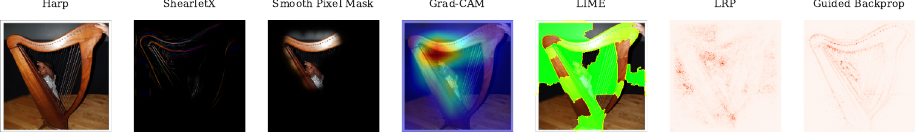}
    \vspace{-.2cm}
    \end{minipage}
  \begin{minipage}{1.0\linewidth}
  \centering
    \includegraphics[width=1.\linewidth]{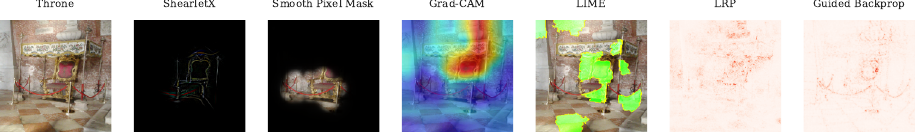}
    \vspace{-.2cm}
    \end{minipage}
  \begin{minipage}{1.0\linewidth}
  \centering
    \includegraphics[width=1.\linewidth]{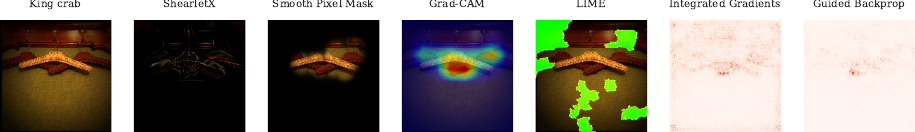}
    \vspace{-.2cm}
    \end{minipage}
      \begin{minipage}{1.0\linewidth}
  \centering
    \includegraphics[width=1.\linewidth]{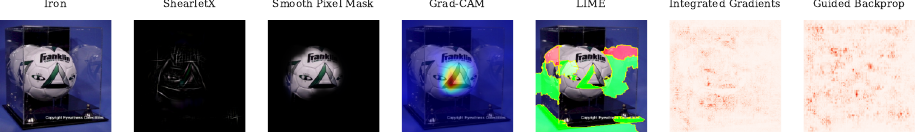}
    \vspace{-.2cm}
    \end{minipage}

   \caption{ShearletX compared to smooth pixel mask \cite{Fong_2019_ICCV}, LIME \cite{Ribeiro2016WhySI}, Grad-CAM \cite{Selvaraju2019GradCAMVE},  LRP \cite{Layerwise_relevance_prop2015}, Integrated Gradients \cite{Integrated_gradient_2017_sundararajan}, and Guided Backprop \cite{guided_backprop_2015}. First two examples are correctly classified by VGG \cite{vgg19} and last two examples are misclassified by MobilenetV3Small \cite{mobilenetv3}. For the harp, note that ShearletX is the only method that effectively seperates the harp from the child playing the harp, indicating, the harp can also be correctly classified without a human playing the harp. For the throne, we observe that ShearletX is able to seperate the throne from other decorations. Here, other methods such as guided backprop, seem to operate more like an edge detector and highlight many other edges such as the floor tiles. Smooth methods such as Grad-CAM  and the smooth pixel mask  give very rough localizations of part of the throne. Finally, for the misclassifications, which we find more challenging to explain, we can see that only ShearletX can effectively expose the crab and the iron that the classifier saw in the hangers and soccer ball, respectively.}
   \label{fig:general comparison}
\end{figure*}

\subsubsection*{General Saliency Map Comparison}
Our experimental results proved that ShearletX has an advantage over the state of the art mask explanation by Fong et al. \cite{Fong_2019_ICCV}.  Other saliency map methods \cite{Layerwise_relevance_prop2015,guided_backprop_2015, Integrated_gradient_2017_sundararajan, Selvaraju2019GradCAMVE} assign a relevance score to each pixel, allowing to order pixels by relevance. Such methods need to be quantitatively validated post-hoc with metrics such as the area over the perturbation curve \cite{Arras2017ExplainingRN} or the pointing game \cite{Fong_2019_ICCV}. It is challenging to meaningfully compare ShearletX on such metrics, since (1) we cannot order the features in ShearletX by relevance due to the binary nature of masks and (2) in ShearletX the mask is in shearlet space and not in pixel space. Nevertheless, in Supplementary Material \ref{appendix: quantitative comparison} we add a quantitative comparison. In Figure \ref{fig:general comparison}, we compare ShearletX qualitatively to pixel attribution methods to demonstrate the insights ShearletX can give that more heuristic pixel attribution methods cannot.

\section{Conclusion}\label{sec:conclusion}
We presented ShearletX, a novel mask explanation method, and two explanation metrics (hallucination score and conciseness-preciseness score) to evaluate mask explanations. ShearletX is more effective than other methods at separating fine-detail patterns, which are relevant for the classifier, from nearby nuisance patterns, that do not affect the classifier. Our theoretical results and experiments show ShearletX is well-protected against explanation artifacts and delivers superior explanations than previous mask explanation methods. Our examples illustrate cases when ShearletX can meaningfully interpret classifications but pixel attribution methods cannot.
In the future, we will focus on improving the runtime of ShearletX, which is currently slower than other mask explanation methods, to provide real-time explanations of excellent quality.

{\small
\bibliographystyle{ieee_fullname}
\bibliography{egbib}
}

\newpage

$ $
\newpage

\appendix

\section{Theory}\label{Appendix Theory}
In this section, we give the supplementary material for our theoretical results. We formally proof Theorem \ref{thm:shearletx} and Theorem \ref{thm:waveletx} from the main paper that show masking in the shearlet or wavelet space cannot create artificial edges for continuous images.

\subsection{ShearletX Theoretical Result}
We begin by recounting the definition of the wavefront set, which is a good model edges  for continuous images, particularly when working with shearlets.
\begin{definition}{\cite[Section 8.1]{AnLinPDOHoermander}}\label{def:WaveFrontSet}
Let $f \in L^2(\R^2)$ and $k \in \N$. A point $(x, \lambda)\in \R^2 \times \mathbb{S}^1$ is a \emph{$k$-regular directed point of $f$} if there exist open neighbourhoods $U_x$ and $V_\lambda$ of $x$ and $\lambda$ respectively and a smooth function $\phi\in C^\infty(\R^2)$ with $\supp \phi \subset U_x$ and $\phi(x) = 1$ such that
\begin{equation*}
    \bigl| \widehat{\phi f} (\xi) \bigr|
    \leq C_k \bigl( 1+|\xi| \bigr)^{-k}
    \quad \text{$\forall$ $\xi \in \R^2 \setminus \{0\}$ s.t. $\xi /|\xi| \in V_\lambda$}
\end{equation*}
holds for some $C_k >0$. The \emph{$k$-wavefront set} $\textrm{WF}_k(f)$ is the complement of the set of all $k$-regular directed points and the \emph{wavefront set} $\textrm{WF}(f)$ is defined as
\[
\textrm{WF}(f) \coloneqq \bigcup_{k \in \N} \textrm{WF}_k(f),
\]
\end{definition}
The wavefront set is completely determined by the decay properties of the shearlet transform, which is formalized in the following Theorem from \cite{grohstheorem_wavefrontset}.
\begin{theorem}[Theorem 2 in \cite{grohstheorem_wavefrontset}]
Let $\psi$ be a Schwartz function with infinitely many vanishing moments in $x_1$-direction. Let $f$ be a tempered distribution and $D=D_1\cup D_2$, where \begin{align*}
    &D_1 = \{(t_0,s_0) \in\R^2 \times [-1,1]:\; \textrm{for}\\
    &\textrm{$(s,t)$ in a neighborhood $U$ of $(s_0,t_0)$,}\\
    &\textrm{$|\mathcal{SH}_\psi(f)(a,s,t)|= O(a^k)$, for all $k\in\N$,}\\
    &\textrm{with the implied constant uniform over $U$})\}
\end{align*}
and
\begin{align*}
    &D_2 = \{(t_0,s_0) \in\R^2 \times (1,\infty]:\; \textrm{for}\\
    &\textrm{$(1/s,t)$ in a neighborhood $U$ of $(s_0,t_0)$,}\\
    &\textrm{$|\mathcal{SH}_\psi(f)(a,s,t)|= O(a^k)$, for all $k\in\N$,}\\
    &\textrm{with the implied constant uniform over $U$})\}.
\end{align*}
Then 
\begin{equation}
    \textrm{WF}(f)^c = D.
\end{equation}
\end{theorem}

For the following theorem, we model the edges in the image $x$ by the wavefront set $\textrm{WF}(x)$.

\begin{theorem}\label{thm:shearletx appendix}
Let $x\in L^2[0,1]^2$ be an image modeled as a $L^2$-function.  Let $m$ be a mask on the shearlet coefficients of $x$ and let $\hat x$ be the image $x$ masked in shearlet space with $m$. Then, we have  $\textrm{WF}(\hat x)\subset \textrm{WF}(x)$ and thus masking in shearlet space did not create new edges.
\end{theorem}
\begin{proof}
Note that the shearlet transform is invertible. Hence, we have
by definition of $\hat x$ 
\begin{align}
    \mathcal{SH}(\hat x)(a,s,t) = \mathcal{SH}(x)(a,s,t)\cdot m(a,s,t).\label{proof:equation 1}
\end{align}
To show $\textrm{WF}(\hat x)\subset \textrm{WF}(x)$, it suffices to show $\textrm{WF}^c(\hat x)\supset \textrm{WF}^c(x)$. Let $(t,s)\in  \textrm{WF}^c(x)$ be arbitrary with $|s|<1$. Then, by definition of the wavefront set, we have for all $N\in\N$
\begin{equation}
    |\mathcal{SH}(x)(a,s,t)| = O(a^N)
\end{equation}
for $a\to0$.
Since $m(a,s,t)\in[0,1]$, we also have for all $N\in\N$
\begin{align}
    &|\mathcal{SH}(\hat x)(a,s,t)| = |\mathcal{SH}(x)(a,s,t)|\cdot |m(a,s,t)|\\
    &\leq |\mathcal{SH}(x)(a,s,t)|  = O(a^N).
\end{align}
This implies $(t,s)\in \textrm{WF}^c(\hat x)$. Thus, we showed the claim $\textrm{WF}^c(\hat x)\supset \textrm{WF}^c(x)$.
\end{proof}

\subsection{WaveletX Theoretical Result}

When analyzing WaveletX, we opt to model singularities via local Lipschitz regularity (see Definition \ref{def: lipschitz regularity}) instead of using the wavefront set approach. This approach is preferable since the Lipschitz regularity of a function is completely characterized by the rate of decay of its wavelet coefficients, as the scale goes to zero (see Theorem \ref{thm:wavelet decay}).

\begin{definition}[Lipschitz Regularity]\label{def: lipschitz regularity}
A function $f:\R^2\to \R$ is \emph{uniformly Lipschitz} $\alpha\geq 0$ over a domain $\Omega\subset \R^2$ if there exists $K>0$, such that for any $v\in \Omega$ one can find a polynomial $p_v$ of degree $\lfloor\alpha \rfloor$ such that
\begin{align}
    \forall x\in\Omega,\;\; |f(x)-p_v(x)|\leq K|x-v|^{\alpha}.
\end{align}
The \emph{Lipschitz regularity} of $f$ over $\Omega$ is the supermum over all $\alpha$, for which $f$ is uniformly Lipschitz $\alpha$ over $\Omega$.
The infimum of K, which satisfies the above equation, is the \emph{homogenous Hölder} $\alpha$ norm $\|f\|_{\tilde C^\alpha}$. 
\end{definition}

\begin{theorem}[Theorem 9.15 \cite{mallat_sparse_edition}]\label{thm:wavelet decay}

Let $x\in L^2[0,1]^2$ be a continuous image with Lipschitz regularity $\alpha\geq 0$. Then there exist $B\geq A>0$ such that for all $J\in\Z$ we have
\begin{align}
    A \|x\|_{\tilde C^\alpha} \leq \sup_{1\leq l\leq 3, j\leq J, 2^jn\in[0,1)^2} \frac{|\langle x, \psi^l_{j,n}\rangle|}{ 2^{j(\alpha+1)}} \leq B\|x\|_{\tilde C^\alpha}.\nonumber
\end{align}
\end{theorem}
In our Theorem \ref{thm:wavelet decay}, we will heavily rely on the connection between Lipschitz regularity and wavelet decay that is formalized in Theorem \ref{thm:wavelet decay}. As preparation for our result, we first give a corollary to Theorem \ref{thm:wavelet decay} that shows a function is uniformly Lipschitz $\alpha$ if and only if the wavelet coefficients decay faster than $\mathcal{O}(2^{j(\alpha+1)})$ for $j\to-\infty$.
\begin{corollary}\label{corollary to mallats theorem}
Let $a,b\in\R$ with $a<b$ and consider a continuous image $x\in L^2[a,b]^2$ with square domain $[a,b]^2$.
Then the following two statements are equivalent:
\begin{enumerate}
    \item The function $x$ is uniformly Lipschitz $\alpha$.
    \item There exists a constant $C>0$ such that for all $J\in\Z$ with $J\leq 0$, we have
\begin{equation}
    \sup_{1\leq l\leq 3, j\leq J, 2^jn\in[a,b)^2} \frac{|\langle x, \psi^l_{j,n}\rangle|}{ 2^{j(\alpha+1)}} \leq C.\nonumber
\end{equation}
\end{enumerate}
\end{corollary}
\begin{proof}
We prove the Corollary for $[a,b]^2=[0,1]^2$ and the general case follows simply with a scaling argument. First, we prove (1) implies (2). Suppose the function $x\in L^2[0,1]^2$ is uniformly Lipschitz $\alpha$. Then $x$ has Lipschitz regularity $\alpha^*\geq \alpha$. By Theorem \ref{thm:wavelet decay}, we then obtain a constant $\tilde B>0$ such that for all $J\in Z$
\begin{align}
    \sup_{1\leq l\leq 3, j\leq J, 2^jn\in[0,1)^2} \frac{|\langle x, \psi^l_{j,n}\rangle|}{ 2^{j(\alpha^* +1)}} \leq \tilde B\|x\|_{\tilde C^{\alpha^*}}.\nonumber
\end{align}
We then have 
\begin{align*}
    &\sup_{1\leq l\leq 3, j\leq J, 2^jn\in[0,1)^2} \frac{|\langle x, \psi^l_{j,n}\rangle|}{ 2^{j(\alpha +1)}} \\
    &=\sup_{1\leq l\leq 3, j\leq J, 2^jn\in[0,1)^2} \frac{|\langle x, \psi^l_{j,n}\rangle|}{ 2^{j(\alpha^* +1)}} 2^{j(\alpha^* -\alpha)}\\
    &\leq \tilde B \|x\|_{\tilde C^{\alpha^*}}\sup_{j\leq J}2^{j(\alpha^*-\alpha)}\\
    &=  \tilde B \|x\|_{\tilde C^{\alpha^*}}2^{J(\alpha^*-\alpha)}\\
    &\leq \tilde B \|x\|_{\tilde C^{\alpha^*}}.
\end{align*}
By setting $C\coloneqq \tilde B \|x\|_{\tilde C^{\alpha^*}}$, we have shown (1) implies (2). 
Next, we show that (2) implies (1). Suppose
there exists a constant $C$ such that for all $J\in\Z$ with $J<0$, we have
\begin{equation}
    \sup_{1\leq l\leq 3, j\leq J, 2^jn\in[a,b)^2} \frac{|\langle x, \psi^l_{j,n}\rangle|}{ 2^{j(\alpha+1)}} \leq C.\label{upper bound 1 corrolary}
\end{equation}
We prove $x$ is uniformly Lipschitz $\alpha$ by contradiction. Supposed $x$ has Lipschitz regularity $\beta$ with $0\leq \beta<\alpha$.
We then have by Theorem \ref{thm:wavelet decay} that there exists constants $A>0$ such that for all $J\in\Z$
\begin{align}
    A \|x\|_{\tilde C^\beta} \leq \sup_{1\leq l\leq 3, j\leq J, 2^jn\in[0,1)^2} \frac{|\langle x, \psi^l_{j,n}\rangle|}{ 2^{j(\beta+1)}}. \label{lower bound 1 corrolary}
\end{align}
By taking $J\to -\infty$  in (\ref{lower bound 1 corrolary}), we obtain a sequence of $(j_k)_{k\in\N}$ with $j_k\rightarrow-\infty$ satisfying
\begin{align}
     \frac{A}{2} \|x\|_{\tilde C^\beta}2^{j_k(\beta+1)} \leq |\langle x, \psi^l_{j_k,n}\rangle| \leq C 2^{j_k(\alpha+1)},\label{corollary contradiction inequality}
\end{align}
for all $k\in\N$. 
But this is a contradiction, since for large enough $k\in\N$, $j_k$ is so negative that the upper bound in (\ref{corollary contradiction inequality}) is strictly smaller then the lower bound in (\ref{corollary contradiction inequality}). Thus, $x$ must be uniformly Lipschitz $\alpha$, which finishes the proof.

\end{proof}
Next, we define a Lipschitz regular point of an image as a point for which the image has locally Lipschitz regularity $\alpha$ with $\alpha\geq 1$.
\begin{definition}[Lipschitz Regular Point]\label{def:lipschitz regular point}
Let $x\in L^2[0,1]^2$ be a continuous image. Let $g:\R\to\R$ be a smooth cutoff function satisfying the following properties:
\begin{enumerate}
    \item $\forall t\in\R:\; |t|\leq 1/2\Longrightarrow g(t)=1$
    \item $\forall t\in\R:\; |t|\geq 1 \Longrightarrow g(t)=0$
    \item $\forall t\in\R:\;  |g(t)|\leq 1$
\end{enumerate}
Define the 2d cutoff function $h:\R^2\to\R,$  $ h(t_1,t_2)\coloneqq g(t_1)g(t_2)$.
We say a point $t^* = (t^*_1, t^*_2)\in[0,1]^2$ is a regular point of $x$ if there exists  $0<a \leq 1$ 
such that the localized image $\tilde x: [0,1]^2\to \R$,
\begin{equation}
\tilde x(t_1,t_2)\coloneqq h((t_1-t_1^*)/a, (t_2-t_2^*)/a)\cdot x(t_1,t_2)    \nonumber
\end{equation}
has Lipschitz regularity $\alpha\geq 1$.
\end{definition}
A Lipschitz singular point is any point that is not a Lipschitz regular point. Lipschitz singular points model image elements such as edges and point singularities. 

\begin{theorem}\label{theorem: waveletx formal version}
Let $x\in L^2[0,1]^2$ be an image. Consider an orthonormal  wavelet basis that comprises compactly supported wavelets. Let $m$ be a bounded mask in wavelet space and denote by $\hat x$ the image $x$ masked in wavelet space with $m$. Then, every Lipschitz regular point $t^*$ of $x$ is also a Lipschitz regular point of $\hat x$.
\end{theorem}
\begin{proof}
Let $t^*=(t^*_1,t^*_2)\in[0,1]^2$ be a Lipschitz regular point of $x$. By definition, there exists $0<a\leq 1$, such that the localized image 
\begin{align}
    &\tilde x: [0,1]^2\to\R, \\
    &\tilde x(t_1,t_2)\coloneqq h((t_1-t_1^*)/a, (t_2-t_2^*)/a)\cdot x(t_1,t_2)
\end{align}
has Lipschitz regularity $\alpha\geq 1$, where  $h: \R^2\to\R$  is the smooth cutoff function from Definition \ref{def:lipschitz regular point}. By Theorem \ref{thm:wavelet decay}, there exists a constant $B>0$ such that for every $J\in\mathbb{Z}$
\begin{align}
    \sup_{1\leq l\leq 3, j\leq J, 2^jn\in[0,1)^2} \frac{|\langle \tilde  x, \psi^l_{j,n}\rangle|}{ 2^{j(\alpha+1)}} \leq B\|\tilde  x\|_{\tilde C^\alpha}.\label{eq: wavelet decay upper bound tilde x}
\end{align}
By definition of the smooth cutoff function $h$, we know that $h$ is equal to $1$ on the square $S_a(t^*)$ with side length $a$, centered at $t^*\in[0,1]^2$.
For each $j\in \Z$, we define the set
\begin{equation}
    \Omega_j \coloneqq \Big\{(n_1,n_2)\in\N^2: \supp\psi_{j,n}\subset S_a(t^*)\Big\}\nonumber
\end{equation}
of grid locations at scale $2^j$ where the wavelet support is  contained in $S_a(t^*)$. Note there exists a sufficiently small $J^*\in\Z$, such that $\Omega_j\neq \emptyset$, for all $j\leq J_0$. Fix such a $J^*\in\Z$. Then, for all $j\leq J^*$ and $n\in\Omega_j$, we have 
\begin{align}
\forall t\in \supp \psi_{j,n}^l:\; \tilde x(t) = x(t)\label{eq: supp of wavelets is subset}
\end{align}
because $\tilde x$ is obtained as $x$ times a  cutoff function, which is equal to 1 on the square $S_a(t^*) \supset \supp \psi_{j,n}^l\ni t$.
Since the wavelet system is assumed to be an orthonormal basis, the wavelet coefficients of $\hat{x}$ are equal to the masked wavelet coefficients of $x$, namely $\langle  \hat x, \psi^l_{j,n}\rangle=m_{j,n}\langle  x, \psi^l_{j,n}\rangle$, where $m_{j,n}$ is the mask entry for the wavelet coefficient with parameters $(j,n)$. 
We then have
\begin{align}
    &\sup_{1\leq l\leq 3, j\leq J^*, n\in\Omega_j} \frac{|\langle  \hat x, \psi^l_{j,n}\rangle|}{ 2^{j(\alpha+1)}}\nonumber\\
    &= \sup_{1\leq l\leq 3, j\leq J^*, n\in\Omega_j} \frac{|m_{j,n}||\langle  x, \psi^l_{j,n}\rangle|}{ 2^{j(\alpha+1)}}.\nonumber
\end{align}
Without loss of generality we can assume $\sup_{j,n} |m_{j,n}|\leq 1$, otherwise we would just add a constant factor to the analysis. We obtain
\begin{align}
    &\sup_{1\leq l\leq 3, j\leq J^*, n\in\Omega_j} \frac{|\langle  \hat x, \psi^l_{j,n}\rangle|}{ 2^{j(\alpha+1)}}\\
    &\leq\sup_{1\leq l\leq 3, j\leq J^*, n\in\Omega_j} \frac{|\langle  x, \psi^l_{j,n}\rangle|}{ 2^{j(\alpha+1)}}\\
    &=\sup_{1\leq l\leq 3, j\leq J^*, n\in\Omega_j}\frac{|\langle \tilde x, \psi^l_{j,n}\rangle|}{ 2^{j(\alpha+1)}},
\end{align}
where we used property (\ref{eq: supp of wavelets is subset}) for the last equality. We further upper bound the expression by taking the supremum over a larger set of indices:
\begin{align}
    & \sup_{1\leq l\leq 3, j\leq J^*, n\in\Omega_j}\frac{|\langle \tilde x, \psi^l_{j,n}\rangle|}{ 2^{j(\alpha+1)}}\\
    &\leq  \sup_{1\leq l\leq 3, j\leq J^*, 2^jn\in[0,1)^2}\frac{|\langle \tilde x, \psi^l_{j,n}\rangle|}{ 2^{j(\alpha+1)}}\leq B\|\tilde x\|_{\tilde C^\alpha},
\end{align}
where we used the upper bound on the wavelet decay from (\ref{eq: wavelet decay upper bound tilde x}) for the last inequality. Overall, we showed
\begin{align}
    \sup_{1\leq l\leq 3, j\leq J^*, n\in\Omega_j} \frac{|\langle  \hat x, \psi^l_{j,n}\rangle|}{ 2^{j(\alpha+1)}} \leq B\|\tilde x\|_{\tilde C^\alpha}. \label{upper bound on subset of wavelet coeffs}
\end{align}
Next, choose $a^\prime\coloneqq a/2$ and consider the smaller square $S_{a^\prime}(t^*)\subset S_a(t^*)$ of side length $a^\prime$ which is centered at $t^*$.
There exists a sufficiently small scale $J_0\in\Z$ so that wavelets with scale parameter $j\leq J_0$ whose support intersects $ S_{a^\prime}(t^*)$ must be contained in $S_a(t^*)$. Namely, for all $j\leq J_0$ and $n\in\Z^2$, we have
\begin{equation}
 \supp \psi_{j,n}\cap S_{a^\prime}(t^*)\neq \emptyset\Longrightarrow \supp \psi_{j,n} \subset S_a(t^*).\label{intersection implies contained property}
\end{equation}
Next, we project $\hat x$ to have only scales smaller than $2^{J_0}$ with the projection operator $P_{J_0}$:
\begin{align}
    P_{J_0}\hat x\coloneqq \sum_{j\leq J_0}\sum_{2^jn\in[0,1)^2}\langle \hat x, \psi^l_{j,n} \rangle \psi^l_{j,n}.
\end{align} 
 We show next that $P_{J_0}\hat x$ is uniformly Lipschitz $\alpha\geq 1$ on $S_{a^\prime}(t^*)$. 
 We  have, for every $J\in\mathbb{Z}$,
\begin{align}
    &\sup_{1\leq l\leq3, j\leq J, 2^jn\in S_{a^\prime}(t^*)} \frac{|\langle  P_{J_0}\hat x, \psi^l_{j,n}\rangle|}{ 2^{j(\alpha+1)}} \label{inequalityPJ01}\\
     &\leq \sup_{1\leq l\leq 3, j\leq J_0,n\in\Omega_j} \frac{|\langle  P_{J_0}\hat x, \psi^l_{j,n}\rangle|}{ 2^{j(\alpha+1)}}\label{inequalityPJ02} \\
     & =\sup_{1\leq l\leq 3, j\leq J_0, n\in\Omega_j} \frac{|\langle  \hat x, \psi^l_{j,n}\rangle|}{ 2^{j(\alpha+1)}}\label{inequalityPJ03}\\
     &\leq B\|\tilde x\|_{\tilde C^\alpha}\label{inequalityPJ04},
\end{align} 
where we used in the equality (\ref{inequalityPJ02}) that the wavelet coefficients of $P_{J_0}\hat x$ for scale parameters $j> J_0$ are zero and equation (\ref{intersection implies contained property}). In equality  (\ref{inequalityPJ03}), we used that $\langle P_{J_0}\hat x, \psi_{j,n}^l \rangle=\langle \hat x, \psi_{j,n}^l \rangle$ for all $j\leq J_0$, and for the last inequality  (\ref{inequalityPJ04}) we used the inequality in (\ref{upper bound on subset of wavelet coeffs}) where $J^*$ can be chosen as $J_0$. 
We can apply now Corollary \ref{corollary to mallats theorem} to $P_{J_0}\hat x$,  which shows that $P_{J_0}\hat x$ is uniformly Lipschitz $\alpha\geq 1$ on the domain $S_{a^\prime}(t^*)$. 
The Lipschitz $\alpha$ property is determined by the asymptotics of the wavelet coefficients for scales going to $0$. Therefore, if the projection $P_{J_0}\hat x$ is uniformly Lipschitz $\alpha$ on the domain $S_{a^\prime}(t^*)$ then so is $\hat x$ uniformly Lipschitz $\alpha$ on the domain $S_{a^\prime}(t^*)$. Finally, we show that $t^*$ is a regular point of $\hat x$. 
We take a sufficiently small scaling factor $a''$ with $0<a''<a^\prime$ so that the cutoff function
\begin{align*}
    (t_1,t_2)\mapsto h((t_1-t^*_1)/a'', (t_2-t_2^*)/a'')
\end{align*}
has support contained in $S_{\alpha^\prime}(t^*)$. The localized image
\begin{align}
     h\big((t_1-t^*_1)/a'', (t_2-t_2^*)/a''\big)\cdot \hat x(t_1,t_2)
\end{align}
is then a product of a uniformly Lipschitz $\alpha$ image with a smooth cut-off function, and is hence a uniformly Lipschitz $\alpha$ function with regularity $\geq\alpha$. 
Hence, $t^*$ is a regular point of $\hat x$, which finishes the proof.

\end{proof}

\section{Experiments}\label{Appendix Experiments}
In this section, we give the supplementary material for our experiments in Section \ref{sec:experiments}.
\subsection{Implementation Details}\label{Appendix: implementation details}
We implemented our methods and experiments in PyTorch \cite{Paszke2019PyTorchAI} and describe the details for each method in the following.
\subsubsection*{ShearletX}
Our implementation of the shearlet transform is an adaptation of the python library  pyShearLab2D\footnote{\url{https://na.math.uni-goettingen.de/pyshearlab/pyShearLab2D.m.html}} to PyTorch. The digital shearlet coefficients of a $3\times256\times256$ image are returned as a $49\times3\times256\times256$  tensor where the first 49 channels capture the discretely sampled scale and shearing parameters of the shearlet transform. To optimize the shearlet mask on the  $49\times3\times256\times256$ tensor, we use the Adam optimizer \cite{adam} with learning rate $10^{-1}$ and for the other Adam parameters we use the PyTorch default setting. The mask is optimized for 300 steps. The expectation in the ShearletX optimization objective 
\begin{align}
    \max_{m} \;&\mathop{\mathbb{E}}_{u\sim \nu}\Big[\Phi_c(\mathcal{DSH}^{-1}(m\odot \mathcal{DSH}(x) + (1-m)\odot u))\Big] \nonumber\\
    \;&- \lambda_1 \|m\|_1 - \lambda_2  \|\mathcal{DSH}^{-1}(m\odot \mathcal{DSH}(x))\|_1, \nonumber
\end{align}
is approximated with a simple Monte Carlo average over 16 samples from $\nu$, which samples uniform noise adapted to each scale and shearing parameter. More precisely, the perturbation for scale $a$ and shearing $s$ is sampled uniformly from $[\mu_{a,s}-\sigma_{a,s}, \mu_{a,s}+\sigma_{a,s}]$ where $\sigma_{a,s}$ and $\mu_{a,s}$ are the empirical standard deviation and mean of the image's shearlet coefficients at scale $a$ and shearing $s$. The mask is initialized with all ones as in \cite{kolek2022cartoon}. For the hyperparameters $\lambda_1$ and $\lambda_2$, we found $\lambda_1=1$ and $\lambda_2=2$ to work well in practice but many other combinations are possible if one desires more or less sparse explanations.  

\subsubsection*{WaveletX}
The discrete wavelet transform (DWT) returns approximation coefficients and detail coefficients. The detail coefficients are parametrized by scale and by orientation (vertical, horizontal, and diagonal). The number of DWT coefficients is the same as the number of pixels and WaveletX optimizes a mask on the DWT coefficients. For the implementation of the DWT, we use the PyTorch Wavelets package\footnote{\url{https://pytorch-wavelets.readthedocs.io/en/latest/readme.html}}. The mask on the DWT coefficients is optimized with the Adam optimizer \cite{adam} with learning rate $10^{-1}$ and for the other Adam parameters we use the PyTorch default setting.  The mask is optimized for 300 steps. The expectation in the WaveletX optimization objective 
\begin{align}
    \max_{m} \;&\mathop{\mathbb{E}}_{u\sim \nu}\Big[\Phi_c(\mathcal{DWT}^{-1}(m\odot \mathcal{DWT}(x) + (1-m)\odot u))\Big] \nonumber\\
    \;&- \lambda_1 \|m\|_1 - \lambda_2  \|\mathcal{DWT}^{-1}(m\odot \mathcal{DWT}(x))\|_1,\nonumber
\end{align}
approximated with a simple Monte Carlo average over 16 samples from $\nu$, which samples uniform noise adapted to each scale of the wavelet coefficients, analogous to ShearletX. More precisely, the perturbation for scale $a$ is sampled uniformly from $[\mu_{a}-\sigma_{a}, \mu_{a}+\sigma_{a}]$ where $\sigma_{a}$ and $\mu_{a}$ are the empirical standard deviation and mean of the image's wavelet coefficients at scale $a$. The mask is initialized with all ones as in ShearletX and in \cite{kolek2022cartoon}. For the hyperparameters, $\lambda_1$ and $\lambda_2$ we found $\lambda_1=1$ and $\lambda_2=10$ work well in practice but many other combinations are possible if one desires more or less sparse explanations.  

\subsubsection*{CartoonX}
For the examples in the CartoonX method from \cite{kolek2022cartoon}, we used the same parameters and procedure as WaveletX but set $\lambda_2=0$. This is because CartoonX and WaveletX only differ in the new spatial penalty that is controlled by $\lambda_2$. 
\subsubsection*{Smooth Pixel Mask}
For the smooth pixel mask method by Fong et al. \cite{Fong_2019_ICCV}, we use the TorchRay\footnote{]https://github.com/facebookresearch/TorchRay} library, which was written by Fong et al. \cite{Fong_2019_ICCV}. The only hyperaparameter for smooth pixel masks is the area constraint, where we use only the values 20\%, 10\%, or 5\%, as did  Fong et al. \cite{Fong_2019_ICCV}.
\subsubsection*{Pixel Mask without Smoothness Constraints} 
The pixel mask method without smoothness constraints has the following optimization objective:
\begin{equation}
    \max_{m\in[0,1]}\; \mathop{\mathbb{E}}_{u\sim\nu} \Big[\Phi_c(x\odot m + (1-m)\odot u)\Big] - \lambda\cdot\|m\|_1,\nonumber
\end{equation}
The mask $m$ on the pixel coefficients is optimized with  the Adam optimizer \cite{adam} with learning rate $10^{-1}$ and for the other Adam parameters we use the PyTorch default setting.  The mask is optimized for 300 steps. The expectation in the optimization objective is approximated with a simple Monte Carlo average over 16 samples from $\nu$, which is chosen as uniform noise
from $[-\sigma + \mu, \mu+\sigma]$, where $\mu$ and $\sigma$ are the empirical mean an standard deviation of the pixel values of the image, as in \cite{kolek2022cartoon}.

\subsubsection*{Edge Detector}
For the edge detector, we use a shearlet-based edge detector, introduced by Reisenhofer et al. in \cite{reisenhofer2019edge} and adapt the imlementation (PyCoShREM\footnote{\url{https://github.com/rgcda/PyCoShREM}} library)  by Reisenhofer et al. in \cite{reisenhofer2019edge} to PyTorch. We used the shearlet-based edge detector because it was able to extract edges more reliably than a Canny edge detector \cite{canny1986computational} and is mathematically well-founded.

\subsection{Runtime}\label{appendix: runtime}
In Table \ref{tab:runtime}, we compare the runtime of ShearletX and WaveletX for the ImageNet classifier MobilenetV3Small \cite{mobilenetv3}   to (1) smooth pixel mask \cite{Fong_2019_ICCV}, (2) pixel attribution methods, such as, Guided Backprop \cite{guided_backprop_2015}, Integrated Gradients \cite{Integrated_gradient_2017_sundararajan}, and Grad-CAM \cite{Selvaraju2019GradCAMVE}, and (3) LIME \cite{Ribeiro2016WhySI}. All mask explanations, \ie, smooth pixel masks, WaveletX, and ShearletX, are much slower than the pixel attribution methods, which only use a single backward pass for the explanation.  ShearletX is roughly $5\times$ slower than smooth pixel masks and WaveletX,  because the shearlet mask has more entries than pixels or wavelet coefficients and the shearlet transform involves more computations. In the future, we are eager to significantly speed-up ShearletX by optimizing our implementation, using a less redundant shearlet system to reduce the numbner of coefficients of the mask, and exploring better initialization strategies for the shearlet mask to obtain faster convergence.
For instance, we hope to train a neural network in the future that outputs mask initializations for ShearletX that lead to faster convergence.

\begin{table}[h]
  \centering
  \begin{tabular}{@{}lc@{}}
    \toprule
    Method & Time \\
    \midrule
    Integrated Gradients  \cite{Integrated_gradient_2017_sundararajan} & 0.31s\\
    Guided Backprop \cite{guided_backprop_2015} & 0.13s\\
    Grad-CAM \cite{Selvaraju2019GradCAMVE} & 0.13s\\
    LIME \cite{Ribeiro2016WhySI} & 5.22s\\
    Smooth Mask \cite{Fong_2019_ICCV} & 11.61s\\
    WaveletX (ours) & 7.99s\\
    ShearletX (ours) & 54.26s\\
    \bottomrule
    \\
  \end{tabular}
  \caption{Computation time for explanation of MobilenetV3Small \cite{mobilenetv3} decision on ImageNet \cite{deng2009imagenet}. It is well-known that mask explanations are more computationally expensive than pixel attribution methods, such as Integrated Gradients \cite{Integrated_gradient_2017_sundararajan}, Grad-CAM \cite{Selvaraju2019GradCAMVE}, and Guided Backprop \cite{guided_backprop_2015}. ShearletX is slower than WaveletX and Smooth Pixel Masks \cite{Fong_2019_ICCV} due to the mask on the shearlet representation being larger and shearlets involving more computations.}
  \label{tab:runtime}
\end{table}

\subsection{Scatter Plots}\label{appendix: scatter plots}
The scatter plots in Figure \ref{fig:scatter-plot} in the main paper compares the hallucination score and conciseness-preciseness score between ShearletX, WaveletX, smooth pixel masks by Fong et al. \cite{Fong_2019_ICCV}, and pixel masks without smoothness constraints. In this section, we provide evidence that our results from Figure \ref{fig:scatter-plot} are consistent across different area constraints for smooth pixel masks and across different classifiers. 
In Figure \ref{fig:scatter plots supp resnet18}, we show the scatter plots for Resnet18 \cite{He2016DeepRL} for the area constrains $5\%$, $10\%$, and $20\%$.  Figure \ref{fig:scatter plots supp mobilenet} shows the same plots for a MobilenetV3Small \cite{mobilenetv3} network.

\subsection{Quantitative Comparison}\label{appendix: quantitative comparison}
Pixel attribution methods, such as Integrated Gradients \cite{Integrated_gradient_2017_sundararajan}, Guided Backprop \cite{guided_backprop_2015}, and Grad-CAM \cite{Selvaraju2019GradCAMVE}, are commonly compared by insertion and deletion curves \cite{rise2018,Samek2017EvaluatingTV, Arras2017ExplainingRN, kolek2022cartoon}, which gradually insert/delete the most relevant pixels and observe the change in class probability. A good insertion curve exhibits a rapid initial increase in class probability and large area under the curve. A good deletion curve exhibits a rapid initial decay and small area under the curve. 
Comparing ShearletX on insertion and deletion curves poses two challenges: (1) ShearletX is given by a mask that is  defined in shearlet space and not in pixel space, as in other methods. (2) ShearletX does not give a proper ordering for the relevance of coefficients due to the binary nature of the mask.

In Figure \ref{fig:perturbation curves in representation}, we compare insertion and deletion curves for ShearletX, where we perturb the most relevant coefficients for ShearletX in \emph{shearlet space}. Insertion and deletion curves are averaged over 50 random ImageNet validation samples and compared on MobilenetV3Small \cite{mobilenetv3}, ResNet18 \cite{He2016DeepRL}, and VGG16 \cite{vgg19}.
ShearletX performs best among compared methods on the \emph{initial part} of the insertion curves in Figure \ref{fig:perturbation curves in representation}, exhibiting a rapid initial increase in probability score. This is what ShearletX was optimized for: keeping very few coefficients that retain the classification decision. However, once ShearletX achieves its peak, coefficients are inserted that were probably marked as zero and not further ordered by the shearlet mask and therefore inserted in arbitrary order. Consequently, the probability score collapses after the peak. Similar behavior is observed for the deletion curve, which initially decays rapidly for ShearletX and then slows down due to the lacking ordering of unselected coefficients. The hyperparameters for ShearletX in the experiments of Figure \ref{fig:perturbation curves in representation} are ($\lambda_1=1$ and $\lambda_2=2$).

In Figure \ref{fig:perturbation curves in pixel space}, we also experiment with a pixel ordering for ShearletX by ordering pixels simply by their magnitude in the ShearletX explanation. Surprisingly, this ordering beats all other compared methods on the insertion curves on two out of three  classifiers that we evaluate. The deletion curves for the pixel ordering of ShearletX are competitive but not outperforming the other methods. For the pixel ordering of ShearletX, we used smaller sparsity parameters ($\lambda_1=0.5$ and $\lambda_2=0.5$) to avoid having too many deleted pixels in ShearletX that cannot be ordered uniquely.

\begin{figure*}[t]
     \centering
     \begin{subfigure}[b]{0.33\textwidth}
    \centering
    \includegraphics[width=.9\linewidth]{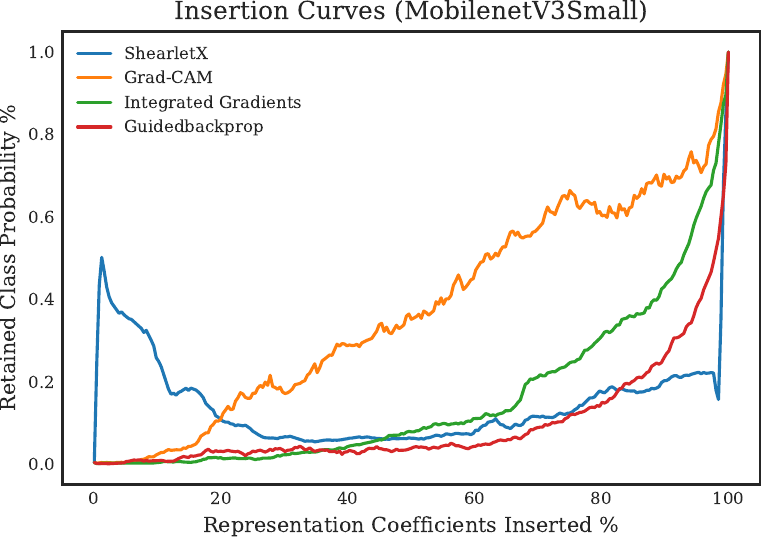}
    \label{fig:my_label}
     \end{subfigure}
     \hfill
     \begin{subfigure}[b]{0.33\textwidth}
    \centering
    \includegraphics[width=.9\linewidth]{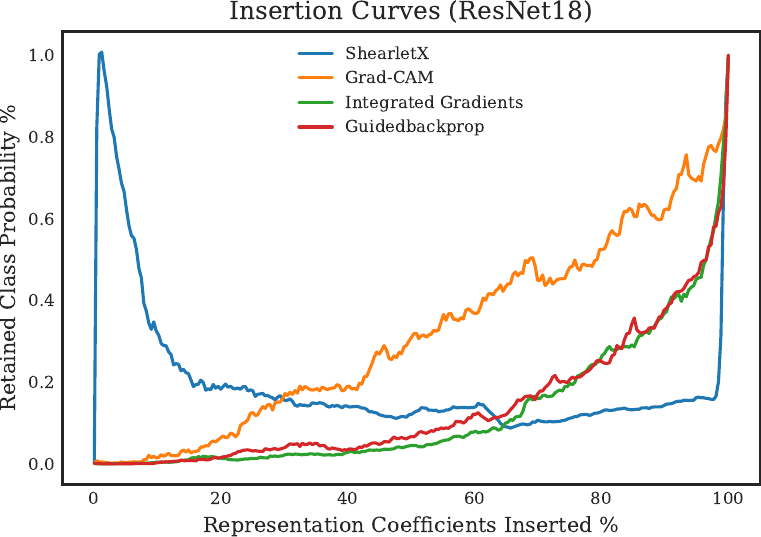}
    \label{fig:my_label}
     \end{subfigure}
     \hfill
    \begin{subfigure}[b]{0.33\textwidth}
    \centering
    \includegraphics[width=.9\linewidth]{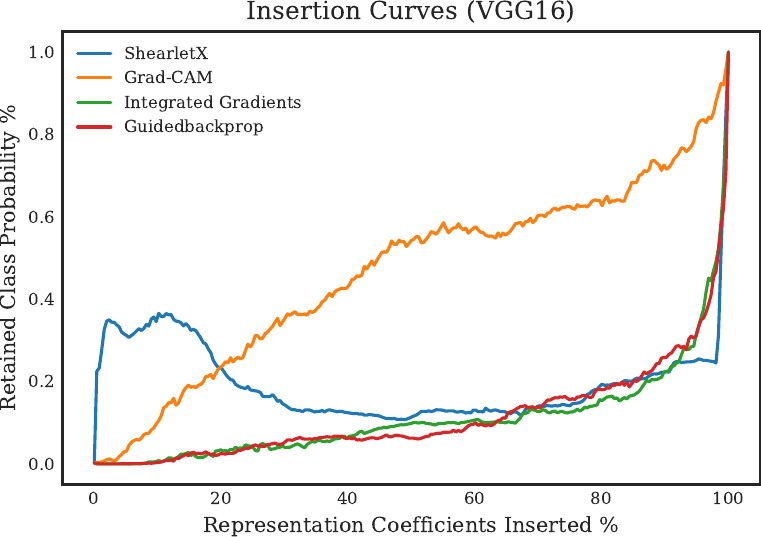}
    \label{fig:my_label}
     \end{subfigure}
     \par\bigskip
      \hfill
    \begin{subfigure}[b]{0.33\textwidth}
    \centering
    \includegraphics[width=.9\linewidth]{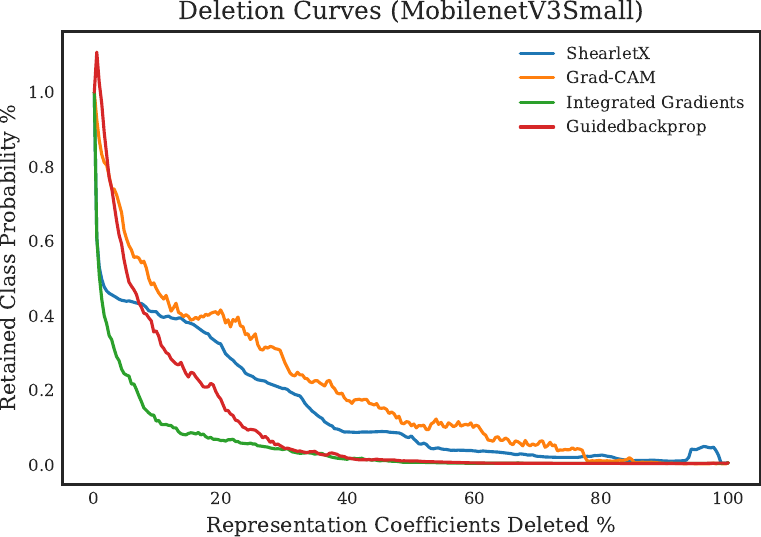}
    \label{fig:my_label}
     \end{subfigure}
     \hfill
     \begin{subfigure}[b]{0.33\textwidth}
    \centering
    \includegraphics[width=.9\linewidth]{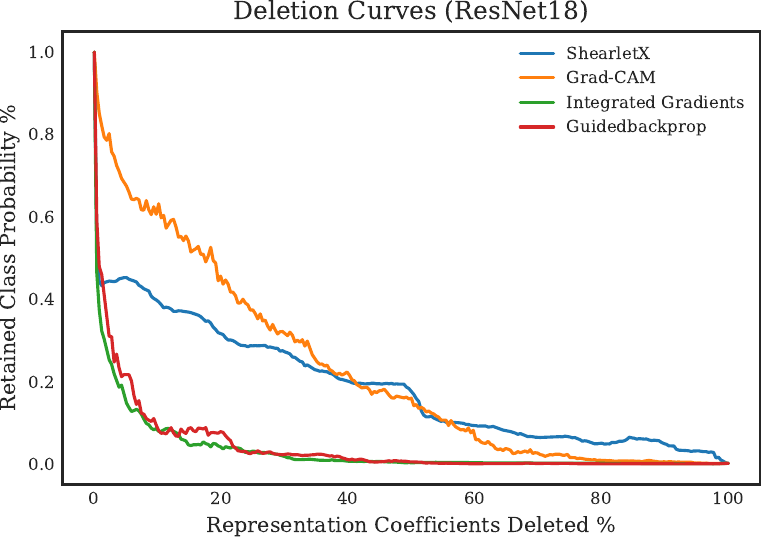}
    \label{fig:my_label}
     \end{subfigure}
     \hfill
     \begin{subfigure}[b]{0.33\textwidth}
    \centering
    \includegraphics[width=.9\linewidth]{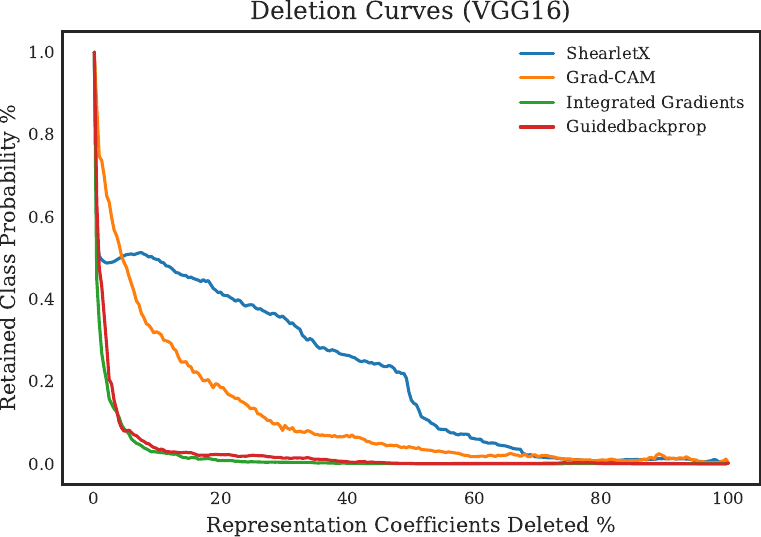}
    \label{fig:my_label}
     \end{subfigure}
        \caption{Insertion and deletion curves for ShearletX  and popular pixel attribution methods (Integrated Gradients \cite{Integrated_gradient_2017_sundararajan}, Grad-CAM \cite{Selvaraju2019GradCAMVE}, and Guided Backprop \cite{guided_backprop_2015}), where representation coefficients are flipped and set to zero. 
        For ShearletX, the representation coefficients are the shearlet coefficients and for all other methods, the pixel coefficients. Insertion curves plot the percentage of inserted representation coefficients (the  most relevant coefficients first) against the retained class probability (class probability after perturbing divided by original class probability). Deletion curves plot the percentage of deleted representation coefficients (the most relevant coefficients first) against the retained class probability. First row: Insertion curves for  MobilenetV3SMall \cite{mobilenetv3}, ResNet18 \cite{He2016DeepRL}, and VGG16 \cite{vgg19}. Second row: Deletion curves for  MobilenetV3SMall \cite{mobilenetv3}, ResNet18 \cite{He2016DeepRL}, and VGG16 \cite{vgg19}.}
        \label{fig:perturbation curves in representation}

\end{figure*}

\begin{figure*}[t]
     \centering
     \begin{subfigure}[b]{0.33\textwidth}
    \centering
    \includegraphics[width=.9\linewidth]{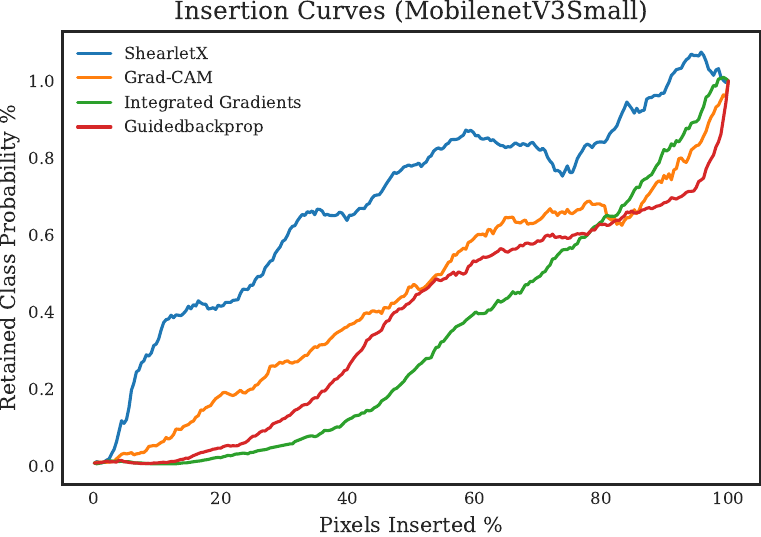}
    \label{fig:my_label}
     \end{subfigure}
     \hfill
     \begin{subfigure}[b]{0.33\textwidth}
    \centering
    \includegraphics[width=.9\linewidth]{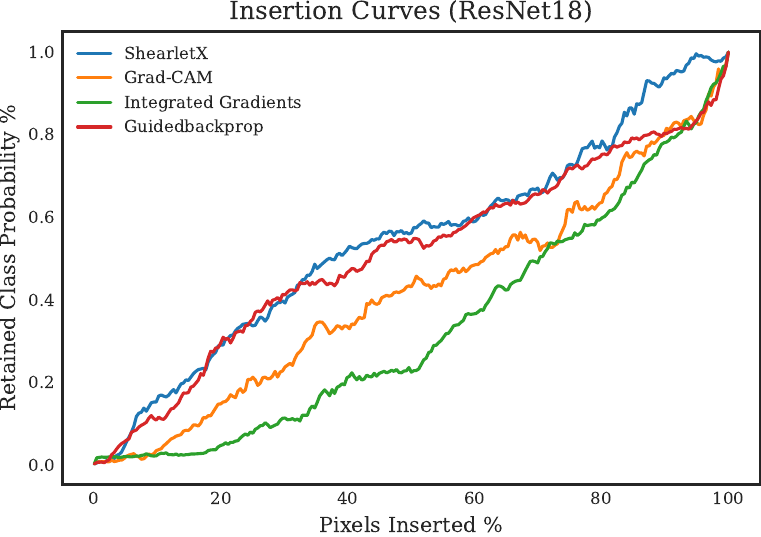}
    \label{fig:my_label}
     \end{subfigure}
     \hfill
    \begin{subfigure}[b]{0.33\textwidth}
    \centering
    \includegraphics[width=.9\linewidth]{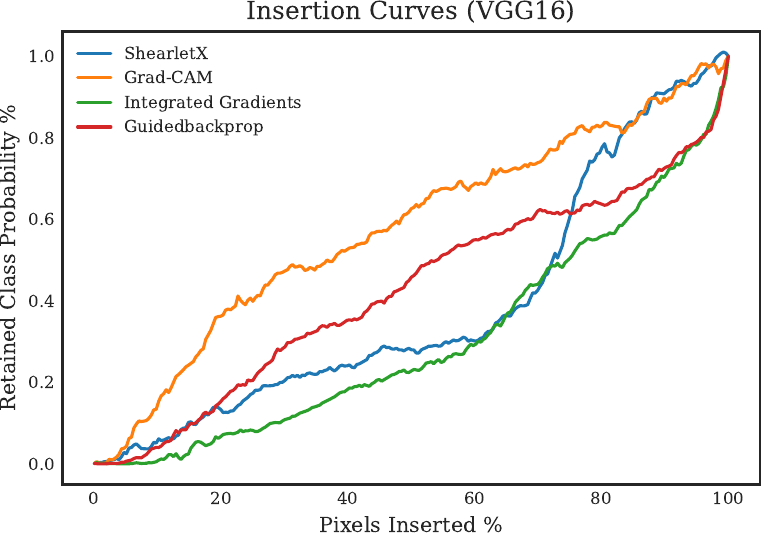}
    \label{fig:my_label}
     \end{subfigure}
     \par\bigskip
      \hfill
    \begin{subfigure}[b]{0.33\textwidth}
    \centering
    \includegraphics[width=.9\linewidth]{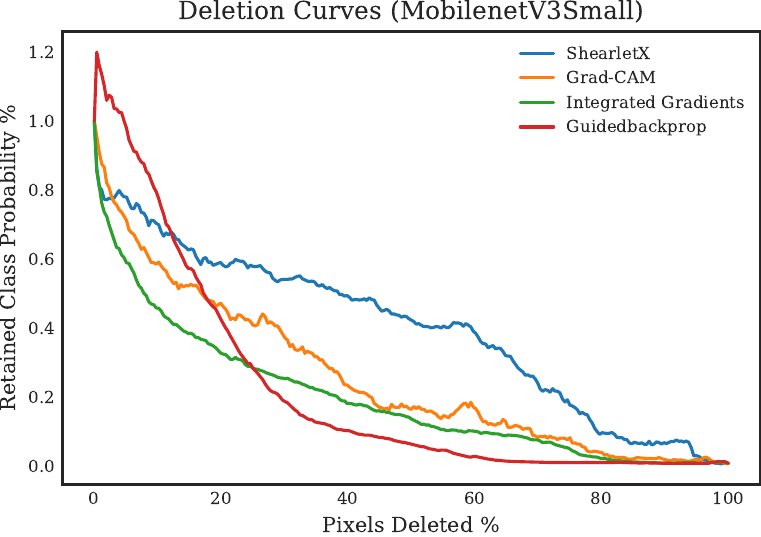}
    \label{fig:my_label}
     \end{subfigure}
     \hfill
     \begin{subfigure}[b]{0.33\textwidth}
    \centering
    \includegraphics[width=.9\linewidth]{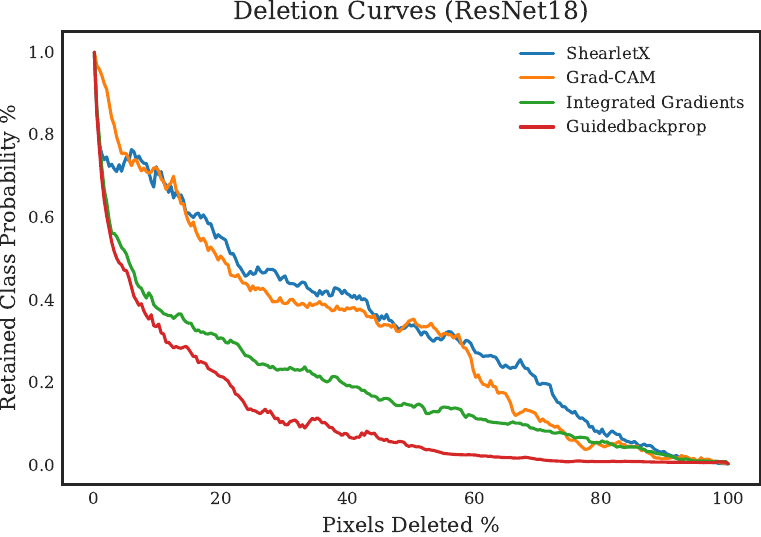}
    \label{fig:my_label}
     \end{subfigure}
     \hfill
     \begin{subfigure}[b]{0.33\textwidth}
    \centering
    \includegraphics[width=.9\linewidth]{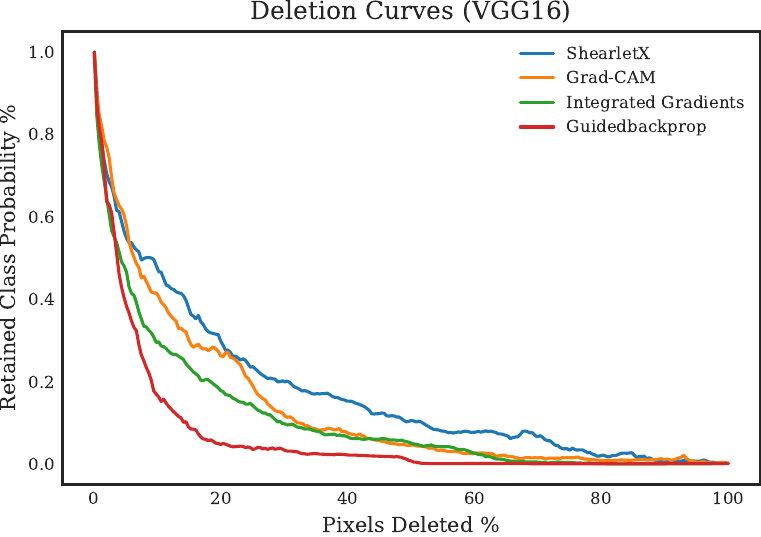}
    \label{fig:my_label}
     \end{subfigure}
        \caption{Insertion and deletion curves for ShearletX  and popular pixel attribution methods (Integrated Gradients \cite{Integrated_gradient_2017_sundararajan}, Grad-CAM \cite{Selvaraju2019GradCAMVE}, and Guided Backprop \cite{guided_backprop_2015}). 
        For ShearletX, we sort pixels by magnitude in the explanation. Insertion curves plot the percentage of inserted pixels  (the most relevant pixels first) against the retained class probability (class probability after perturbing divided by original class probability).
        Deletion curves plot the percentage of deleted pixels  (the most relevant pixels first) against the retained class probability.
        Deleted pixels are replaced with blurred pixel values. First row: Insertion curves for  MobilenetV3SMall \cite{mobilenetv3}, ResNet18 \cite{He2016DeepRL}, and VGG16 \cite{vgg19}. Second row: Deletion curves for  MobilenetV3SMall \cite{mobilenetv3}, ResNet18 \cite{He2016DeepRL}, and VGG16 \cite{vgg19}.}
        \label{fig:perturbation curves in pixel space}

\end{figure*}

\begin{figure*}[h]
     \centering
     \begin{subfigure}[b]{0.33\textwidth}
          \centering
    \includegraphics[width=.9\linewidth]{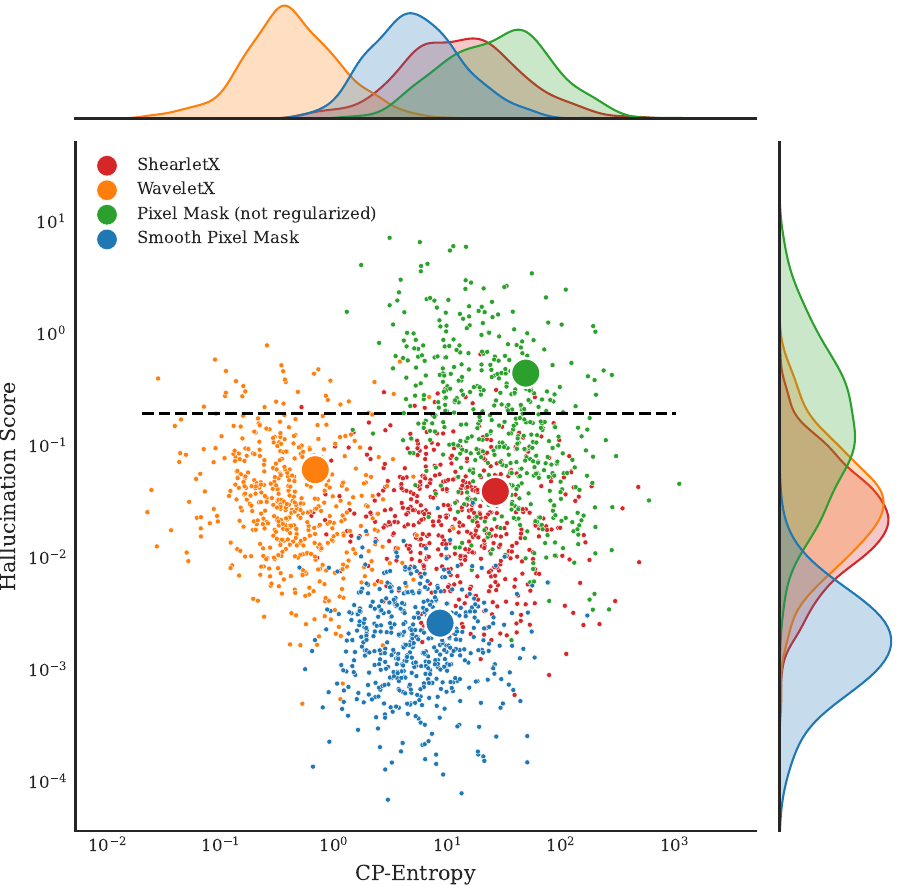}
     \caption{Smooth Pixel Mask Area Constraint: $5\%$}
         \label{fig:y equals x}
     \end{subfigure}
     \hfill
     \begin{subfigure}[b]{0.33\textwidth}
         \centering
    \includegraphics[width=.9\linewidth]{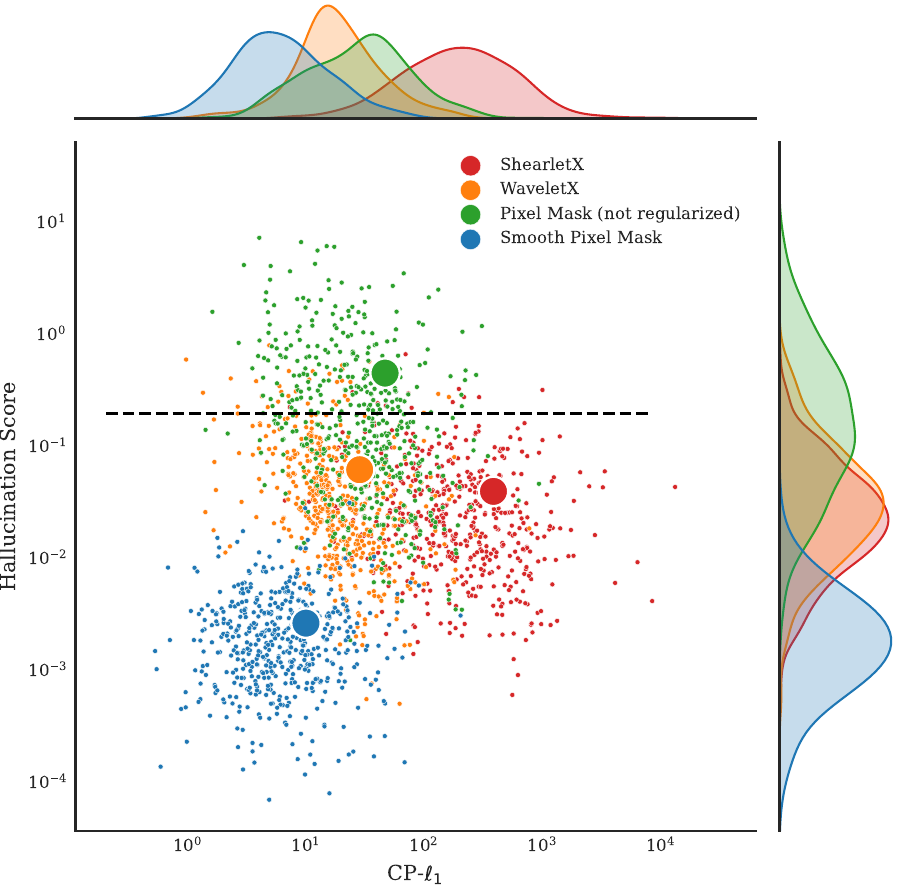}
     \caption{Smooth Pixel Mask Area Constraint: $5\%$}
         \label{fig:three sin x}
     \end{subfigure}
     \begin{subfigure}[b]{0.33\textwidth}
          \centering
    \includegraphics[width=.9\linewidth]{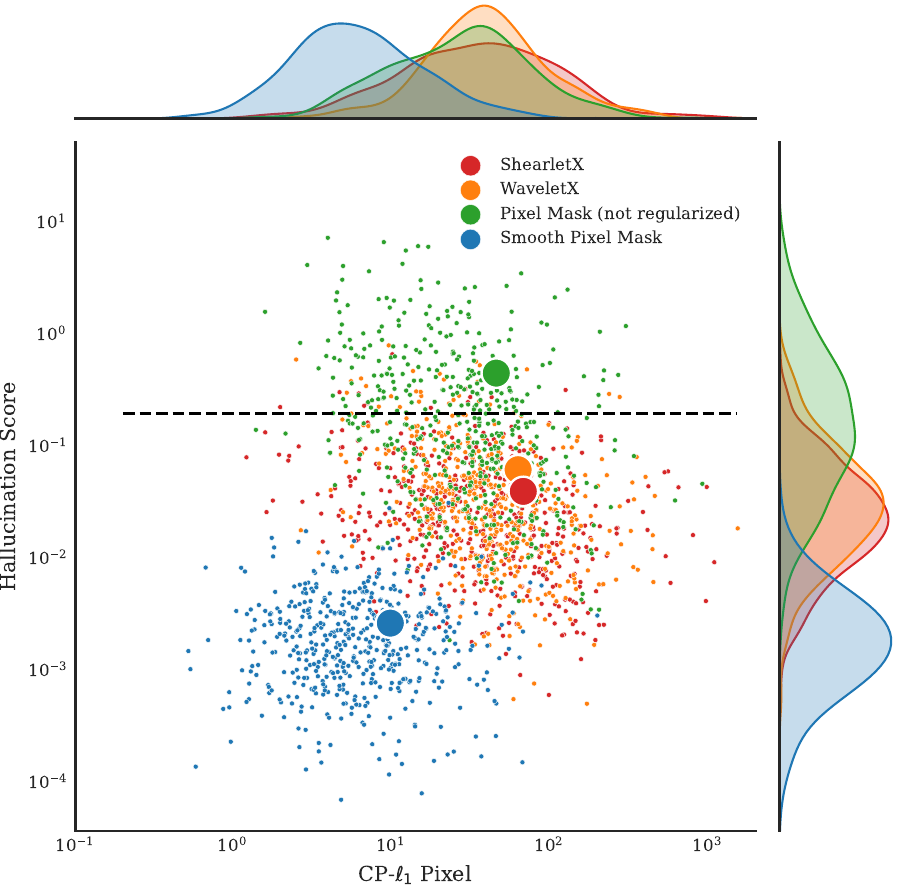}
      \caption{Smooth Pixel Mask Area Constraint: $5\%$}
         \label{fig:y equals x}
     \end{subfigure}
     \hfill
     \begin{subfigure}[b]{0.33\textwidth}
         \centering
    \includegraphics[width=.9\linewidth]{Figures/scatterplot_resnet18_info_as_entropy_in_representation_area_0.1.pdf}
     \caption{Smooth Pixel Mask Area Constraint: $10\%$}
         \label{fig:three sin x}
     \end{subfigure}
          \hfill
     \begin{subfigure}[b]{0.33\textwidth}
         \centering
    \includegraphics[width=.9\linewidth]{Figures/scatterplot_resnet18_info_as_l1_in_representation_area_0.1.pdf}
     \caption{Smooth Pixel Mask Area Constraint: $10\%$}
         \label{fig:three sin x}
     \end{subfigure}
          \hfill
     \begin{subfigure}[b]{0.33\textwidth}
         \centering
    \includegraphics[width=.9\linewidth]{Figures/scatterplot_resnet18_info_as_l1_spatial_area_0.1.pdf}
         \caption{Smooth Pixel Mask Area Constraint: $10\%$}
         \label{fig:three sin x}
     \end{subfigure}
          \hfill
     \begin{subfigure}[b]{0.33\textwidth}
         \centering
    \includegraphics[width=.9\linewidth]{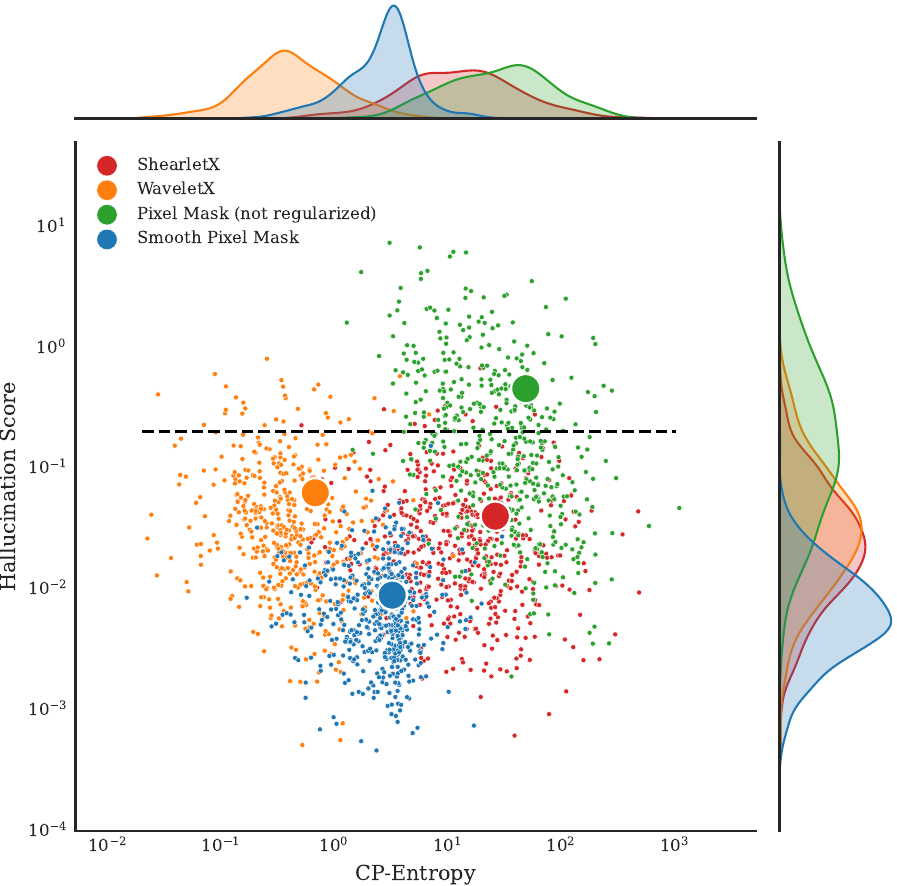}
     \caption{Smooth Pixel Mask Area Constraint: $20\%$}
         \label{fig:three sin x}
     \end{subfigure}
          \hfill
     \begin{subfigure}[b]{0.33\textwidth}
         \centering
    \includegraphics[width=.9\linewidth]{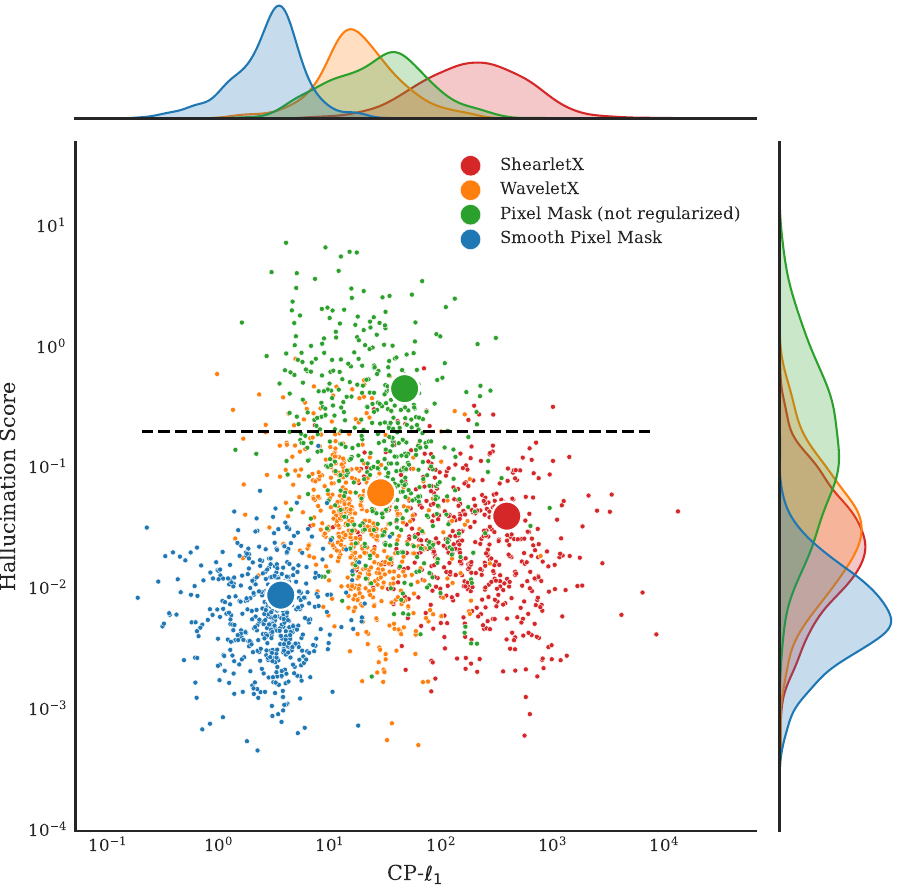}
     \caption{Smooth Pixel Mask Area Constraint: $20\%$}
         \label{fig:three sin x}
     \end{subfigure}
          \hfill
     \begin{subfigure}[b]{0.33\textwidth}
         \centering
    \includegraphics[width=.9\linewidth]{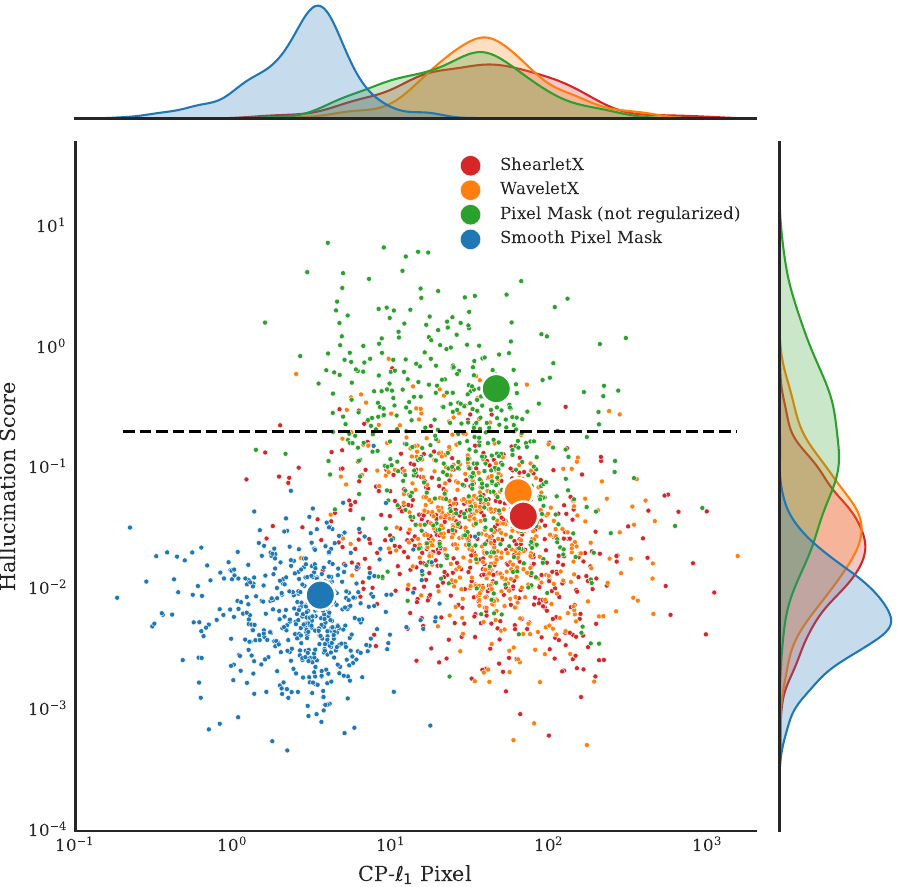}
         \caption{Smooth Pixel Mask Area Constraint: $20\%$}
         \label{fig:three sin x}
     \end{subfigure}

        \caption{Scatter plots of hallucinaton score (lower is better) and conciseness-preciseness score (higher is better) for ShearletX, WaveletX, smooth pixel masks \cite{Fong_2019_ICCV}, and pixel mask without smoothness constraints. We used the classifier ResNet18 \cite{He2016DeepRL} for all scatter plots. First row uses smooth pixel masks \cite{Fong_2019_ICCV} with area constraint $5\%$, second row uses $10\%$, and last row uses $20\%$. The scatter plots shows that the advantage of ShearletX over smooth pixel masks \cite{Fong_2019_ICCV} holds for different area constraints.}
        \label{fig:scatter plots supp resnet18}
\end{figure*}

\begin{figure*}[h]
     \centering
     \begin{subfigure}[b]{0.33\textwidth}
          \centering
    \includegraphics[width=.9\linewidth]{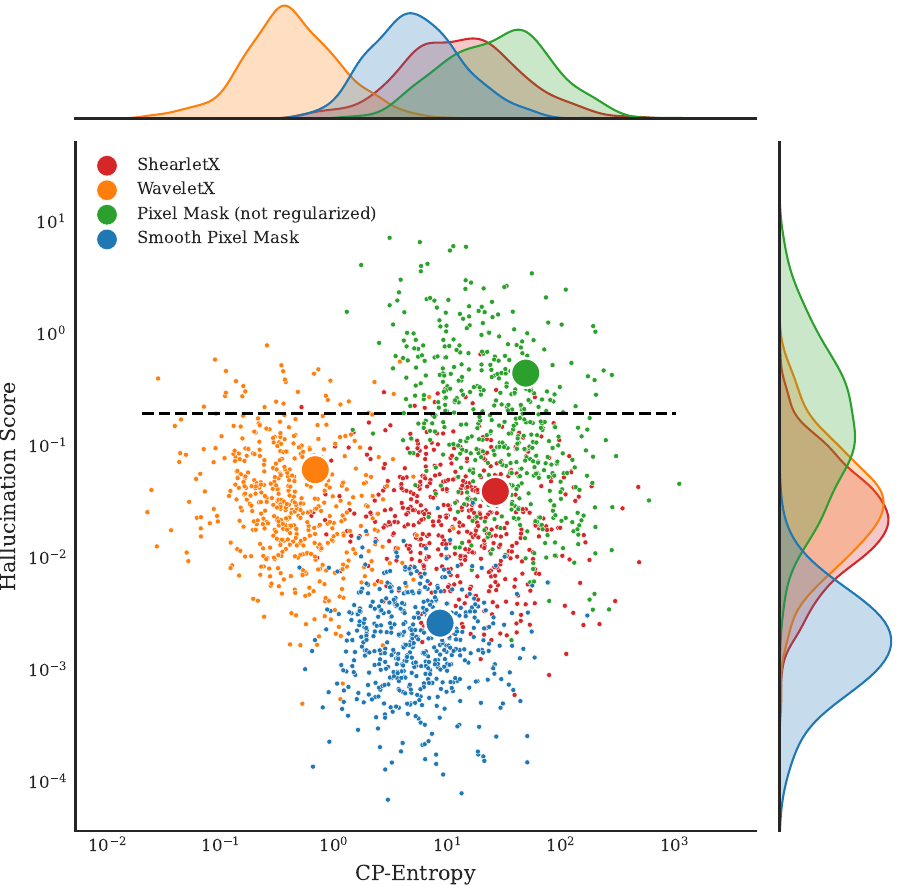}
     \caption{Smooth Pixel Mask Area Constraint: $5\%$}
         \label{fig:y equals x}
     \end{subfigure}
     \hfill
     \begin{subfigure}[b]{0.33\textwidth}
         \centering
    \includegraphics[width=.9\linewidth]{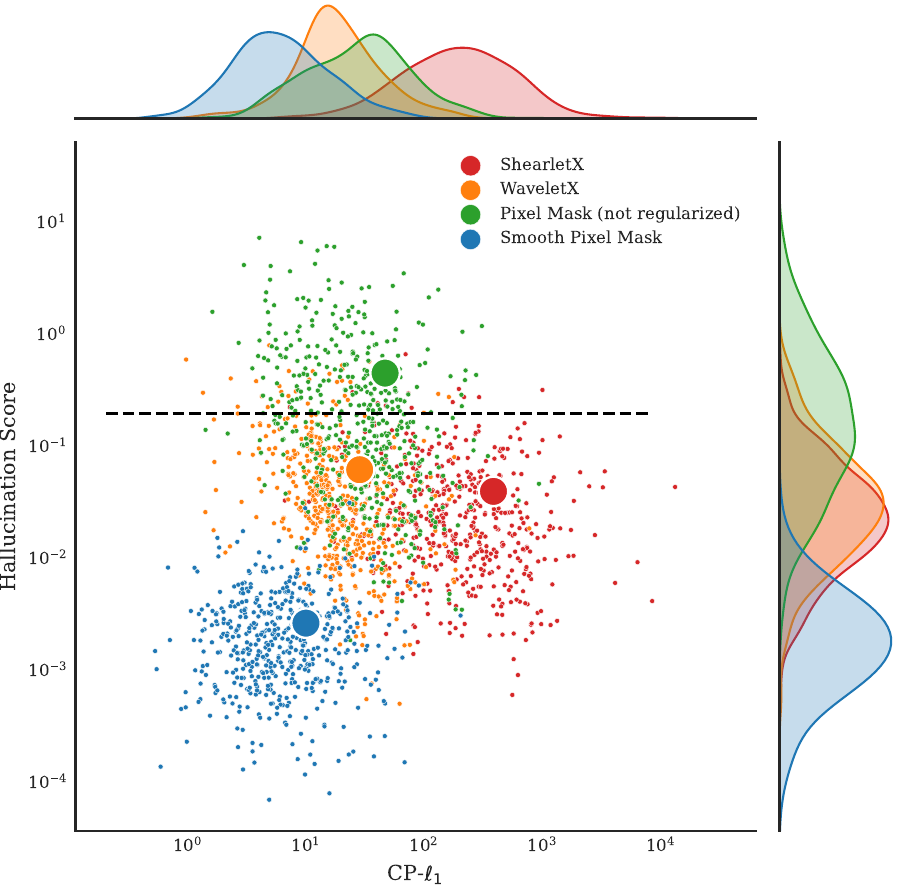}
     \caption{Smooth Pixel Mask Area Constraint: $5\%$}
         \label{fig:three sin x}
     \end{subfigure}
     \begin{subfigure}[b]{0.33\textwidth}
          \centering
    \includegraphics[width=.9\linewidth]{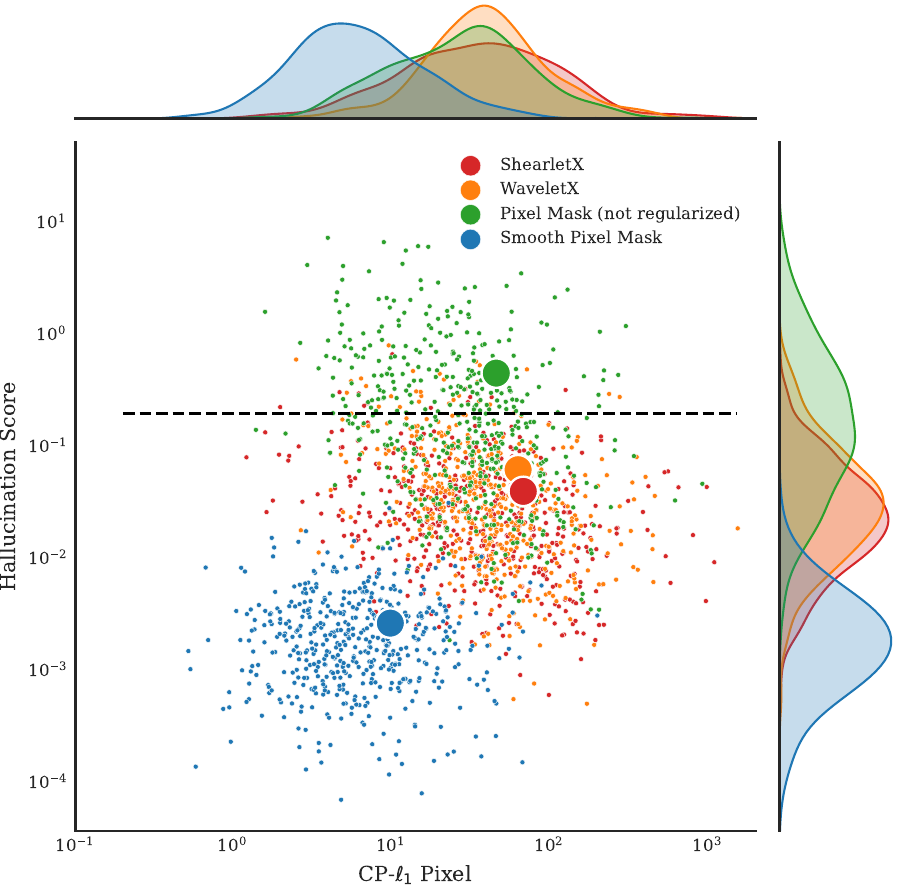}
      \caption{Smooth Pixel Mask Area Constraint: $5\%$}
         \label{fig:y equals x}
     \end{subfigure}
     \hfill
     \begin{subfigure}[b]{0.33\textwidth}
         \centering
    \includegraphics[width=.9\linewidth]{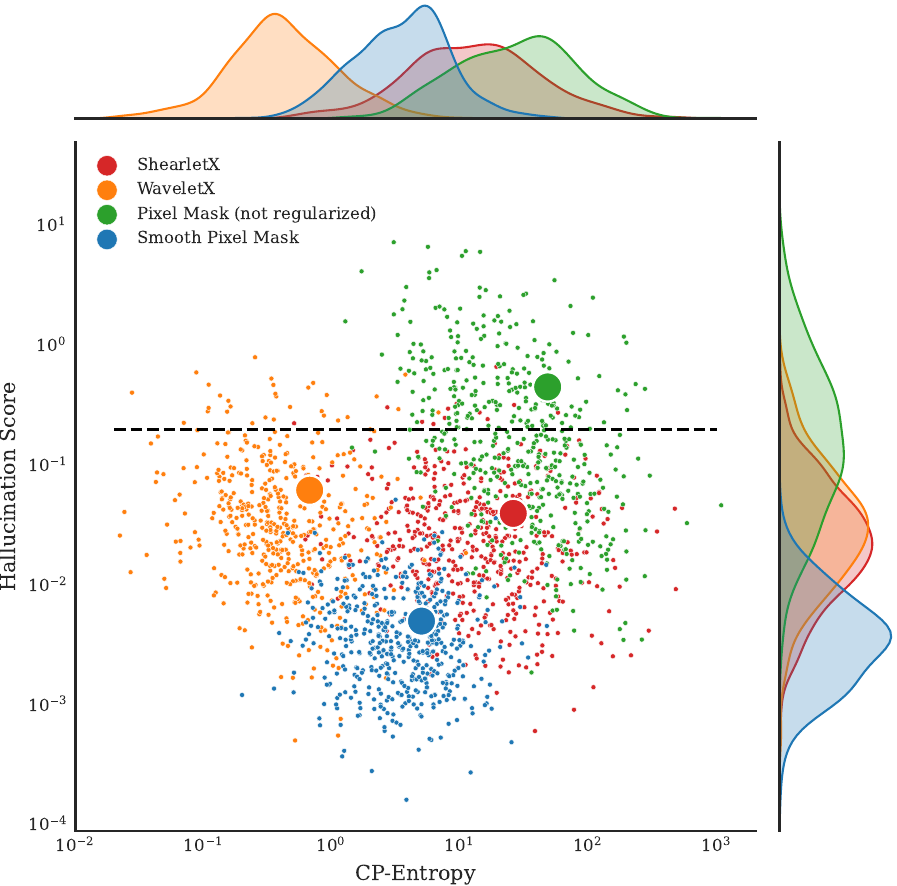}
     \caption{Smooth Pixel Mask Area Constraint: $10\%$}
         \label{fig:three sin x}
     \end{subfigure}
          \hfill
     \begin{subfigure}[b]{0.33\textwidth}
         \centering
    \includegraphics[width=.9\linewidth]{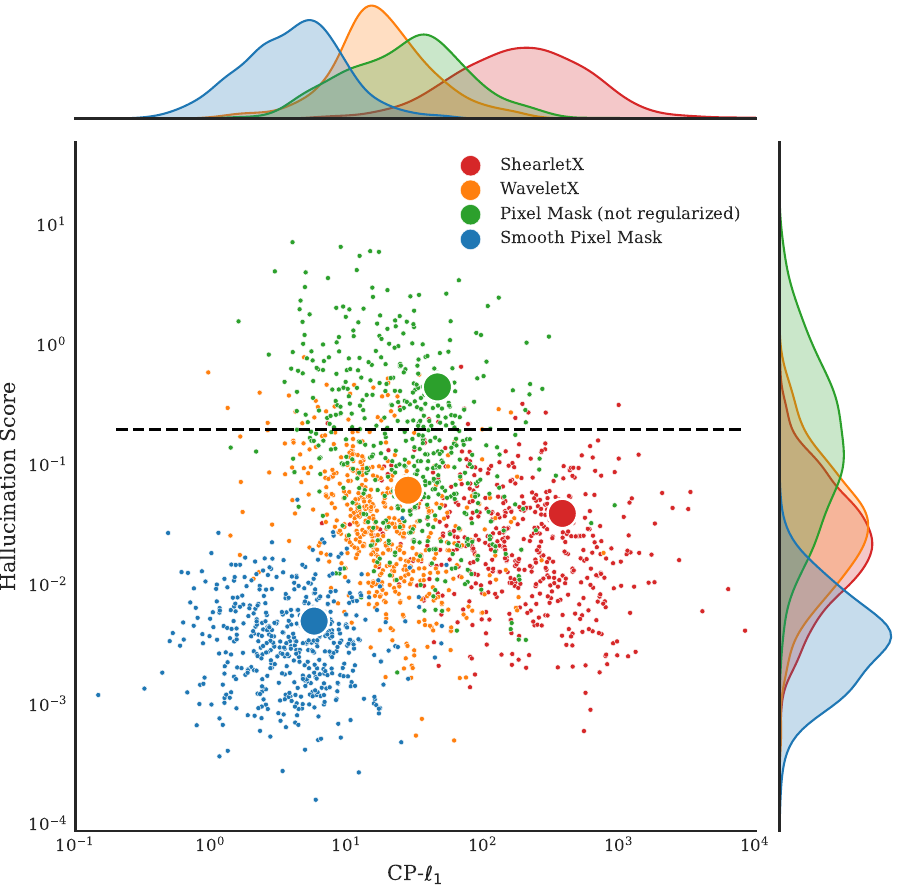}
     \caption{Smooth Pixel Mask Area Constraint: $10\%$}
         \label{fig:three sin x}
     \end{subfigure}
          \hfill
     \begin{subfigure}[b]{0.33\textwidth}
         \centering
    \includegraphics[width=.9\linewidth]{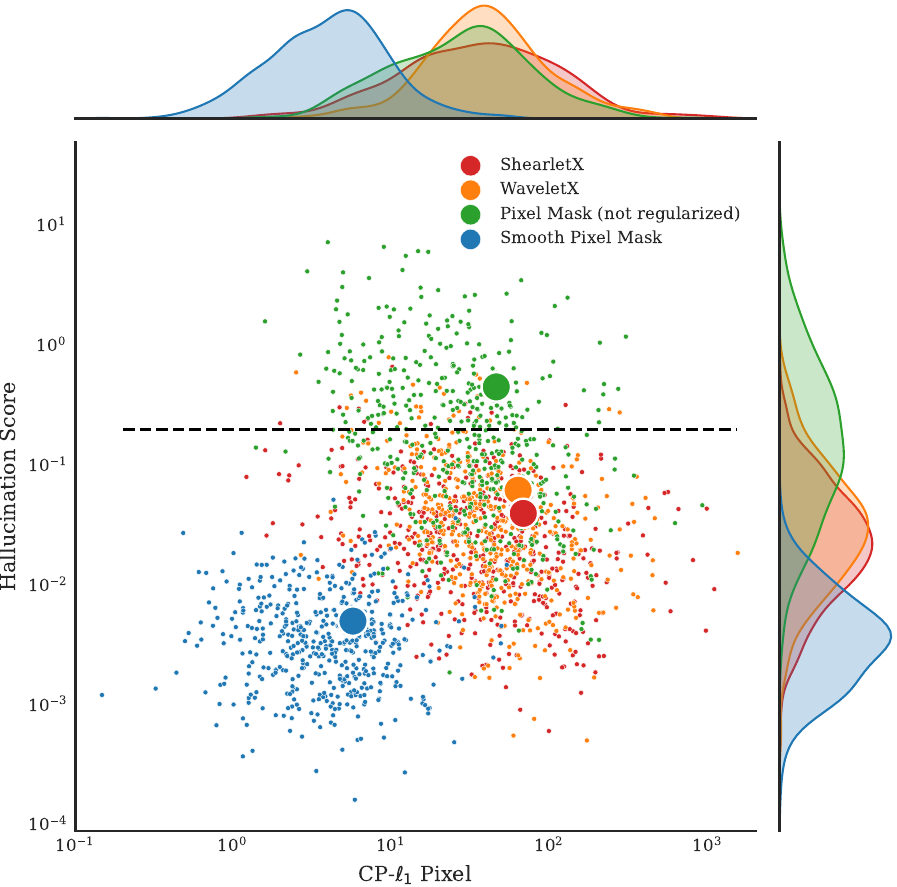}
         \caption{Smooth Pixel Mask Area Constraint: $10\%$}
         \label{fig:three sin x}
     \end{subfigure}
          \hfill
     \begin{subfigure}[b]{0.33\textwidth}
         \centering
    \includegraphics[width=.9\linewidth]{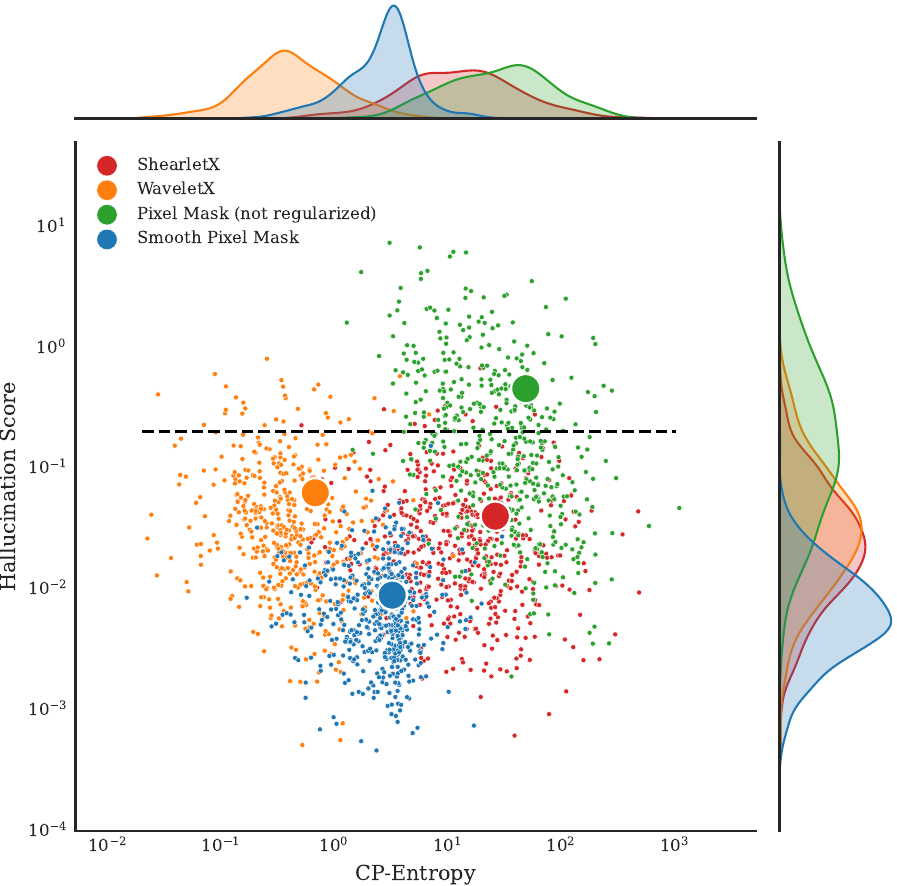}
     \caption{Smooth Pixel Mask Area Constraint: $20\%$}
         \label{fig:three sin x}
     \end{subfigure}
          \hfill
     \begin{subfigure}[b]{0.33\textwidth}
         \centering
    \includegraphics[width=.9\linewidth]{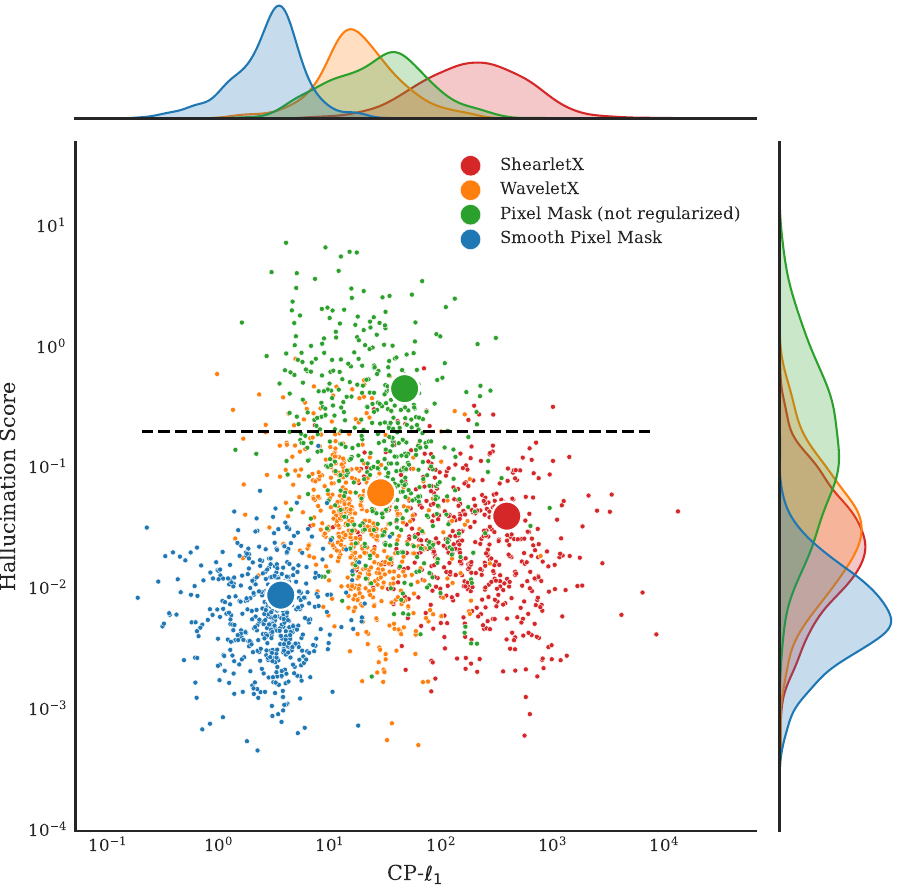}
     \caption{Smooth Pixel Mask Area Constraint: $20\%$}
         \label{fig:three sin x}
     \end{subfigure}
          \hfill
     \begin{subfigure}[b]{0.33\textwidth}
         \centering
    \includegraphics[width=.9\linewidth]{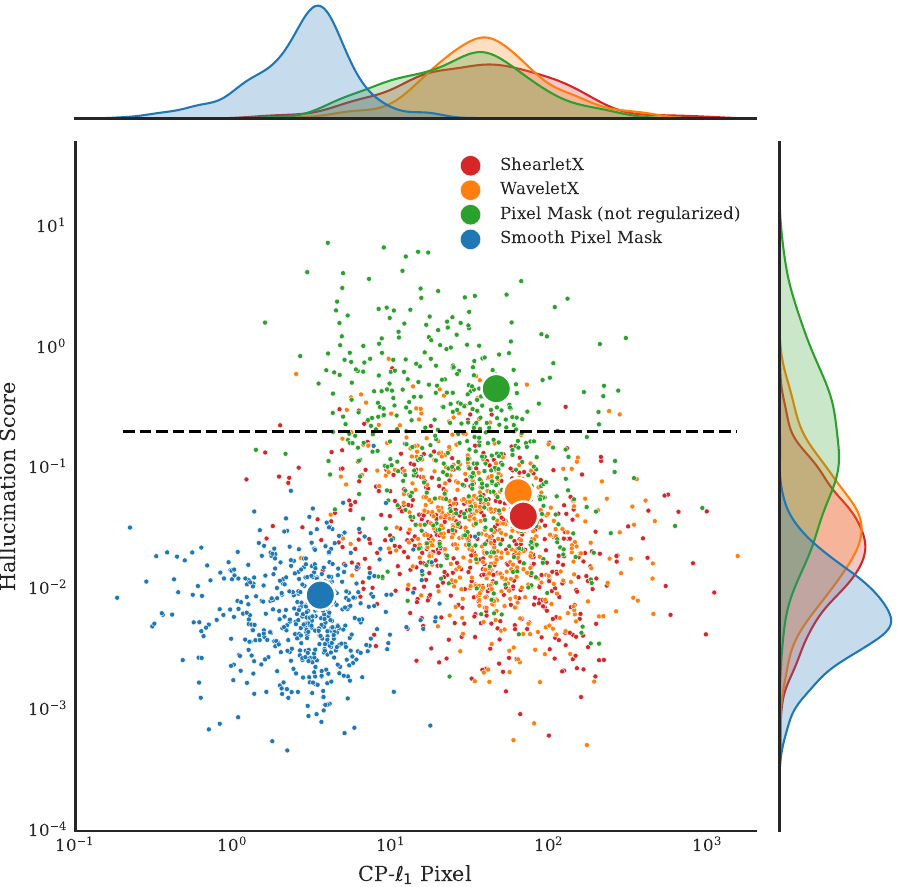}
         \caption{Smooth Pixel Mask Area Constraint: $20\%$}
         \label{fig:three sin x}
     \end{subfigure}
     
        \caption{Scatter plot of hallucinaton score (lower is better) and conciseness-preciseness score (higher is better) for ShearletX, WaveletX, smooth pixel masks \cite{Fong_2019_ICCV}, and pixel mask without smoothness constraints. We used  MobilenetV3Small \cite{mobilenetv3} as a classifier for all scatter plots. First row uses smooth pixel masks \cite{Fong_2019_ICCV} with area constraint $5\%$, second row uses $10\%$, and last row uses $20\%$. The scatter plots shows that the advantage of ShearletX over smooth pixel masks \cite{Fong_2019_ICCV} holds for different area constraints. The scatter plots compared to Figure \ref{fig:scatter plots supp resnet18} also show that the advantage of ShearletX over smooth pixel masks \cite{Fong_2019_ICCV} holds for different classifiers.}
        \label{fig:scatter plots supp mobilenet}
\end{figure*}

\end{document}